\documentclass{article}

     \usepackage[final]{neurips_2019}

\usepackage[utf8]{inputenc} 
\usepackage[T1]{fontenc}    
\usepackage{hyperref}       
\usepackage{url}            
\usepackage{booktabs}       
\usepackage{amsfonts}       
\usepackage{nicefrac}       
\usepackage{microtype}      
\usepackage{amsmath}
\usepackage{amsthm}
\usepackage[algo2e]{algorithm2e}
\usepackage{algorithm}
\usepackage{algorithmicx}
\usepackage[titletoc,title]{appendix}
\usepackage{makecell}
\usepackage{multirow}
\usepackage{array,graphicx}
\usepackage{float}
\usepackage{subcaption} 
\usepackage{xcolor}
\usepackage{caption} 
\usepackage{authblk}

\setcitestyle{numbers,square}
\PassOptionsToPackage{,sort&compress}{natbib}
\DeclareMathOperator{\Tr}{Tr}
\DeclareMathOperator{\sign}{sign}

\newcommand{\tmop}[1]{\ensuremath{\operatorname{#1}}}
\newcommand{\nobracket}{}
\newcommand{\nocomma}{}
\newcommand{\noplus}{}
\newcommand{\tmmathbf}[1]{\ensuremath{\boldsymbol{#1}}}
\newcommand{\nosymbol}{}
\newcommand{\summary}[1]{}
\newcommand{\chieh}[1]{\textcolor{black}{{#1}}}

\newenvironment{packeditemize}{\begin{list}{$\bullet$}{\setlength{\itemsep}{0pt}\addtolength{\labelwidth}{-5pt}\setlength{\leftmargin}{\labelwidth}\setlength{\listparindent}{\parindent}\setlength{\parsep}{0pt}\setlength{\topsep}{3pt}}}{\end{list}}

\newtheorem{theorem}{Theorem}
\newtheorem{proposition}{Proposition}
\newtheorem{corollary}{Corollary}
\newtheorem{definition}{Definition}
\newtheorem{lemma}{Lemma}

\title{Solving Interpretable Kernel Dimension Reduction}

\author{\textbf{Chieh Wu}}
\author{\textbf{Jared Miller}}
\author{\textbf{Yale Chang}}
\author{\textbf{Mario Sznaier}}
\author{\textbf{Jennifer Dy}}
\affil{Electrical and Computer Engineering Dept., Northeastern University, Boston, MA}

\begin{document}

\maketitle

\begin{abstract}

Kernel dimensionality reduction (KDR) algorithms find a low dimensional representation of the original data by optimizing kernel dependency measures that are capable of capturing nonlinear relationships. The standard strategy is to first map the data into a high dimensional feature space using kernels prior to a projection onto a low dimensional space. While KDR methods can be easily solved by keeping the most dominant eigenvectors of the kernel matrix, its features are no longer easy to interpret.  Alternatively, Interpretable KDR (IKDR) is different in that it projects onto a subspace \textit{before} the kernel feature mapping, therefore, the projection matrix can indicate how the original features linearly combine to form the new features. Unfortunately, the IKDR objective requires a non-convex manifold optimization that is difficult to solve and can no longer be solved by eigendecomposition. 
Recently, an efficient iterative spectral (eigendecomposition) method (ISM) has been proposed for this objective in the context of alternative clustering.
However, ISM only provides theoretical guarantees for the Gaussian kernel. This greatly constrains ISM's usage since any kernel method using ISM is now limited to a single kernel. 
This work extends the theoretical guarantees of ISM to an entire family of kernels, thereby empowering ISM to solve any kernel method of the same objective. In identifying this family, we prove that each kernel within the family has a surrogate $\Phi$ matrix and the optimal projection is formed by its most dominant eigenvectors. With this extension, we establish how a wide range of IKDR applications across different learning paradigms can be solved by ISM. To support reproducible results, the source code is made publicly available on  \url{https://github.com/chieh-neu/ISM_supervised_DR}.


\end{abstract}

\section{Introduction}
\label{submission}
The most important information for a given dataset often lies in a low dimensional space
~\cite{niu2011dimensionality,suzuki2010sufficient,jolliffe2011principal,roweis2000nonlinear,wickelmaier2003introduction,tenenbaum2000global,kambhatla1997dimension,fukumizu2009kernel,song2008colored,song2007supervised,vladymyrov2013locally,song2012feature,belkin2003laplacian,elhamifar2013sparse}. Due to the ability of kernel dependence measures for capturing both linear and nonlinear relationships, they are powerful criteria for nonlinear dimensionality reduction (DR)~\cite{scholkopf1998nonlinear,barshan2011supervised}. The standard approach is to first map the data into a high dimensional feature space prior to a projection onto a low dimensional space~\cite{scholkopf1997kernel}. This approach has been preferred because it captures the nonlinear relationship with an established solution, i.e., the most dominant eigenvectors of the kernel matrix. However, since the high dimensional feature space maps the original featues nonlinearly, it is no longer interpretable.
Alternatively, if the projection onto a subspace precedes the feature mapping, the projection matrix can be obtained to inform how the original features linearly combine to form the new features. Exploiting this insight, many formulations have leveraged kernel alignment or Hilbert Schmidt Independence Criterion (HSIC)~\cite{gretton2005measuring} to model this approach~\cite{wu2018iterative,fukumizu2009kernel,barshan2011supervised, masaeli2010transformation,scholkopf1998nonlinear,niu2011dimensionality,gangeh2016semi,chang2017clustering,niu2010multiple,song2012feature}. Together, we refer to these approaches as Interpretable Kernel Dimension Reduction (IKDR).   Unfortunately, this formulation can no longer be solved via eigendecomposition, instead, it becomes a highly non-convex manifold optimization that is computationally expensive.

Numerous approaches have been proposed to solve this complex objective. With its orthogonality constraint, it is a form of optimization on a manifold: i.e., the constraint can be modeled geometrically as a Stiefel or Grassmann manifold \cite{james1976topology,nishimori2005learning,edelman1998geometry}.  Earlier work, \citet{boumal2011rtrmc} propose to recast a similar problem on the Grassmann manifold and then apply first and second-order Riemannian trust-region methods to solve it. \citet{theis2009soft} employ a trust-region method for minimizing the cost function on the Stiefel manifold. \citet{wen2013feasible} later propose to unfold the Stiefel manifold into a flat plane and optimize on the flattened representation. While the manifold approaches perform well under smaller data sizes, they quickly become inefficient when the dimension or sample size increases, which poses a serious challenge to larger modern problems. Besides manifold approaches, \citet{niu2014iterative} propose Dimension Growth (DG) to perform gradient descent via greedy algorithm a column at a time. By keeping the descent direction of the current column orthogonal to all previously discovered columns, DG ensures the constraint compliance. 

The approaches discussed thus far have remained inefficient.
Recently, \citet{wu2018iterative} proposed the Iterative Spectral Method (ISM) for alternative clustering where 
their experiments on a dataset of 600 samples showed that it took DG almost 2 days while ISM finished under 2-seconds \textit{with a lower objective cost}. 
Moreover, ISM retains the ability to use eigendecomposition to solve IKDR. Instead of finding the eigenvectors of kernel matrices, ISM uses a small surrogate matrix $\Phi$ to replace the kernel matrix, thereby allowing for a much faster eigendecomposition.  
Yet, ISM is not without its limitations. Since ISM's theoretical guarantees are specific to Gaussian kernels, repurposing ISM to other kernel methods becomes impractical, i.e., a kernel method of a single kernel significantly limits its flexibility and representational power.


\summary{We generalize ISM to other kernels by generalizing $\Phi$.}
In this paper, we expand ISM's theoretical guarantees to an entire family of kernels,
thereby realizing ISM's potential for a wide range of applications.
Within this family, each kernel is associated with a matrix $\Phi$ where its most dominant eigenvectors form the solution. 
Here, $\Phi$ matrices replace the concept of kernels to serve as an interchangeable component of any applicable IKDR kernel method. 
We further extend the family to kernels that are conic combinations of the ISM family of kernels. Here, we prove that any conic combination of kernels within the family also has an associated $\Phi$ matrix constructed using the respective conic combination of $\Phi$s. 

\summary{We draw the connection between ISM to other applications.}
Empowered by extending ISM's theoretical guarantees to other kernels, we present ISM
as a solution to IKDR problems across several learning paradigms, including supervised DR~\cite{fukumizu2009kernel,barshan2011supervised,masaeli2010transformation}, unsupervised DR ~\cite{scholkopf1998nonlinear,niu2011dimensionality}, semi-supervised DR\cite{gangeh2016semi,chang2017clustering}, and alternative clustering ~\cite{wu2018iterative,niu2010multiple,niu2014iterative}. 
Indeed, we demonstrate how many of these applications can be reformulated into an identical optimization objective which ISM solves, implying a significant role for ISM that has been previously unknown.

\textbf{Our Contributions. } 

\begin{packeditemize}

    \item We generalize the theoretical guarantees of ISM to an entire family of kernels and propose the necessary criteria for a kernel to be a member of the family.
    
    \item We generalize ISM to conic combinations of kernels from the ISM family.
    
    \item We establish that 
    ISM can be used to solve general classes of IKDR learning paradigms.
    
    \item We present experimental evidence to highlight the generalization of ISM to a wide range of learning paradigms under a family of kernels and demonstrate its efficiency in terms of speed and better accuracies compared to competing methods.

\end{packeditemize}

\section{A General Form for Interpretable Kernel Dimension Reduction}
\summary{What types of problems can ISM solve?}
Let $X \in \mathbb{R}^{n \times d}$ be a dataset of $n$ samples with $d$ features and let $Y \in \mathbb{R}^{n \times k}$ be the corresponding labels where $k$ denotes the number of classes. Given $\kappa_X(\cdot, \cdot)$ and $\kappa_Y(\cdot,\cdot)$ as two kernel functions that applies respectively to $X$ and $Y$ to construct kernel matrices $K_X \in \mathbb{R}^{n \times n}$ and $K_Y \in \mathbb{R}^{n \times n}$. Also let $H$ be a centering matrix where $H=I-(1/n) \textbf{1}_n \textbf{1}_n^T$ with $H \in \mathbb{R}^{n \times n}$, $I$ as the identity matrix and $\textbf{1}_n$ as a vector of 1s. 
HSIC measures the nonlinear dependence between $X$ and $Y$ whose empirical estimate is expressed as $\mathbb{H}(X,Y) = \frac{1}{(1-n)^2} \Tr(HK_XHK_Y)$, with $\mathbb{H}(X,Y)=0$ denoting complete independence and $\mathbb{H}(X,Y) \gg 0$ as high dependence~\cite{gretton2005measuring}.  Additional background regarding HSIC is provided in Appendix~\ref{app:about_hsic}. 

A general IKDR problem can be posed as discovering a subspace $W \in \mathbb{R}^{d \times q}$ such that $\mathbb{H}(XW,Y)$ is maximized. Since $W$ induces a reduction of dimension, we can assume that $q < d$. To prevent an unbounded solution, the subspace $W$ is constrained such that $W^TW=I$. Since this formulation has a wide range of applications across different learning paradigms, our work investigates the commonality of these problems and discovers that various learning IKDR paradigms can be expressed as the following optimization problem:




\begin{equation}
    \underset{W}{\max} \Tr ( \Gamma K_{X W}) \hspace{0.4cm} \text{s.t.} \hspace{0.2cm} W^T W = I,
     \label{eq:obj_1}
\end{equation}



where $\Gamma$ is a symmetric matrix commonly derived from $K_Y$. 
Although this objective is shared among many IKDR problems, the highly non-convex objective continues to pose a serious challenge. Therefore the realization of ISM's ability to solve Eq.~(\ref{eq:obj_1}) 
impacts many applications. Here, we provide several examples of this connection. 

\textbf{Supervised Dimension Reduction. } In supervised DR \cite{barshan2011supervised,masaeli2010transformation}, both the data $X$ and the label $Y$ are known. We wish to discover a low dimensional subspace $W$ such that $XW$ is maximally dependent on $Y$ \textit{in the nonlinear high dimensional feature space}. This problem can be cast as maximizing the HSIC between $XW$ and $Y$ where we maximize $\Tr ( K_{X W} H K_Y H)$. Since $HK_YH$ includes all known variables, they can be considered as a constant $\Gamma = HK_YH$. 
Eq. (\ref{eq:obj_1}) is obtained by rotating the trace terms and constraining $W$ to $W^TW=I$.

\textbf{Unsupervised Dimension Reduction. } 
\citet{niu2011dimensionality} introduced a DR algorithm for spectral clustering based on an HSIC formulation.  In unsupervised DR, we also discover a low dimensional subspace $W$ such that $XW$ is maximally dependent on $Y$. Therefore, the objective here is actually identical to the supervised objective of $\Tr ( K_{X W} H K_Y H)$, except since $Y$ in unknown here, both $W$ and $Y$ need to be learned. By setting $K_Y = YY^T$,  this problem can be solved by alternating maximization between $Y$ and $W$. When $W$ is fixed, the problem reduces down to spectral clustering~\cite{niu2014iterative} and $Y$ can be solved via eigendecomposition as shown in \citet{niu2011dimensionality}. When $Y$ is fixed, the objective becomes the supervised formulation previously discussed.

%

\textbf{Semi-Supervised Dimension Reduction. } In semi-supervised DR clustering problems \cite{chang2017clustering}, some form of scores $\hat{Y} \in \mathbb{R}^{n \times r}$ are provided by subject experts for each sample. It is assumed that if two samples are similar, their scores should also be similar. In this case, the objective is to cluster the data given some supervised guidance from the experts. The clustering portion can be accomplished by spectral clustering \cite{von2007tutorial} and HSIC can capture the supervised expert knowledge. By simultaneously maximizing the clustering quality of spectral clustering and the HSIC between the data and the expert scores, this problem is formulated as
\begin{align}
  \underset{W,Y}{\max}\, & Tr (Y^T \mathcal{L}_W Y) + \mu \Tr (K_{XW} H K_{\hat{Y}} H),
   \label{eq:ssdr_1}
    \\
    \text{s.t} & \qquad \mathcal{L}_W = D^{-\frac{1}{2}}K_{XW} D^{-\frac{1}{2}}, W^T W = I, Y^T Y = I
\end{align}
where $\mu$ is a constant to balance the importance between the first 
and the second terms of the objective, $D\in\mathbb{R}^{n\times n}$ is 
the degree matrix that is a diagonal matrix with its diagonal elements defined as $D_{diag}= K_{XW}\textbf{1}_n$. 
Similar to the unsupervised DR problem, this objective is solved by alternating optimization of $Y$ and $W$. 
Since the second term does not include $Y$, when $W$ is fixed, the objective reduces down 
to spectral clustering. 
By initializing $W$ to an identity matrix, $Y$ is initialized to the solution of spectral clustering. When $Y$ is fixed, $W$ can be solved by isolating $K_{XW}$. If we let $\Psi = HK_{\hat{Y}}H$ and $\Omega=D^{-\frac{1}{2}}Y Y^T D^{-\frac{1}{2}}$, maximizing Eq. (\ref{eq:ssdr_1}) is equivalent to maximizing $\Tr [(\Omega + \mu \Psi) K_{XW} ]$ subject to $W^TW=I$. At this point, it is easy to see that by setting $\Gamma=\Omega+\mu \Psi$, the problem is again equivalent to Eq. (\ref{eq:obj_1}).

\textbf{Alternative Clustering. } In alternative clustering \cite{niu2014iterative}, a set of labels $\hat{Y} \in \mathbb{R}^{n \times k}$ is provided as the original clustering labels. 
The objective of alternative clustering is to discover an alternative set of labels that is high in clustering quality while different from the original label. In a way, this is a form of semi-supervised learning. Instead of having extra information about the clusters we desire, the supervision here indicates what we wish to avoid. Therefore, this problem can be formulated almost identically as a semi-supervised problem with
\begin{align}
    \underset{W,Y}{\max}\, & \Tr (Y^T \mathcal{L}_W Y) - 
    \mu \Tr (K_{XW} H K_{\hat{Y}} H),
   \label{eq:ac_1}
    \\
    \text{s.t} & \qquad \mathcal{L}_W = D^{-\frac{1}{2}}K_{XW} D^{-\frac{1}{2}}, W^T W = I, Y^T Y = I.
\end{align}
Given that the only difference here is a sign change before the second term, this problem can be solved identically as the semi-supervised DR problem and the sub-problem of maximizing $W$ when $Y$ is fixed can be reduced into Eq.~(\ref{eq:obj_1}).

\section{Extending the Theoretical Guarantees to a Family of Kernels}
\summary{How the ISM algorthm work.}
\textbf{The ISM algorithm. } The ISM algorithm, as proposed by \citet{wu2018iterative}, solves Eq.~(\ref{eq:obj_1}) by setting the $q$ most dominant eigenvectors of a special matrix $\Phi$ as its solution $W$; we define $\Phi$ in a later section. We denote these eigenvectors in our context as $V_{\text{max}}$ and their eigenvalues as $\Lambda$. Since $\Phi$ derived by \citet{wu2018iterative} is a function of $W$, the new $W$ is used to construct the next $\Phi$ which we again set its $V_{\text{max}}$ as the next $W$. This process iterates until the change in $\Lambda$ between each iteration falls below a predefined threshold $\delta$. To initialize the first $W$, the 2nd order Taylor expansion is used to approximate $\Phi$ which yields a matrix $\Phi_0$ that is independent of $W$. We supply extra detail in Appendix~\ref{app:convergence_criteria} and its pseudo-code in Algorithm~\ref{alg:ism}.

    \begin{algorithm}[H]
    \footnotesize
     \textbf{Input :} Data $X$, kernel, Subspace Dimension $q$\\ \textbf{Output :} Projected subspace $W$ \\
     \textbf{Initialization :} Initialize $\Phi_0$ using Table~\ref{table:init_phis}.\\ Set $W_0$ to $V_{\text{max}}$ of $\Phi_0$.\\
     \While{$||\Lambda_i - \Lambda_{i-1}||_2/||\Lambda_i||_2 < \delta$ }{
        Compute $\Phi$ using Table~\ref{table:phis}\\
        Set $W_k$ to $V_{\text{max}}$ of $\Phi$
     }
     \caption{ISM Algorithm}
     \label{alg:ism}
    \end{algorithm}    

\summary{Why ISM is difficult to generalize and how our version is different.}
\textbf{Extending the ISM Algorithm.} Unfortunately, the theoretical foundation of ISM is specifically tailored to the Gaussian kernel. Since the proof relies heavily on the exponential structure of the Gaussian function, extending the algorithm to other kernels seems unlikely. However, we discovered that there exists a family of kernels where each kernel possesses its own distinct pair of $\Phi/\Phi_0$ matrices. From our proof, we discovered a general formulation of $\Phi/\Phi_0$ for any kernel within the family. Moreover, since the only change is the $\Phi/\Phi_0$ pair, the ISM algorithm holds by simply substituting the appropriate $\Phi/\Phi_0$ matrices based on the kernel. We have derived several examples of $\Phi_0$/$\Phi$ in Tables~\ref{table:init_phis} and \ref{table:phis} respectively and supplied the derivation \textit{for each kernel} in Appendices~\ref{app:deriv_phi_0} and \ref{app:deriv_phi}.

\chieh{To clarify the notations for Tables~\ref{table:init_phis} and \ref{table:phis}, given a matrix $\Psi$, we define $D_\Psi$ and $\mathcal{L}_{\Psi}$ respectively as the degree matrix and the Laplacian of $\Psi$ where $D_\Psi=\text{Diag}(\Psi 1_n)$ and $\mathcal{L} = D_\Psi - \Psi$. Here, Diag is a function that places the elements of a vector into the diagonal of a zero squared matrix. While $K_{XW}$ is the kernel computed from $XW$, we denote $K_{XW,p}$ as specifically a polynomial kernel of order $p$. We also denote the symbol $\odot$ as a Hadamard product between matrices.}


\begin{table}[h]   
\begin{minipage}{2.4in}
    \footnotesize
      \begin{tabular}{c|l}
        Kernel & Approximation of $\Phi$s\\
        \midrule
            Linear
            	& $\Phi_0=X^T\Gamma X$ \\
            Squared
            	& $\Phi_0=X^T \mathcal{L}_{\Gamma} X$ \\  
            Polynomial
            	& $\Phi_0=X^T \Gamma X$\\
            Gaussian
            	& $\Phi_0=-X^T \mathcal{L}_{\Gamma} X$\\
            Multiquadratic
            	& $\Phi_0=X^T\mathcal{L}_{\Gamma} X$\\
            \bottomrule       
      \end{tabular}
      \caption{Equations for the approximate \\$\Phi$s for the common kernels.}
      \label{table:init_phis}
\end{minipage}
\begin{minipage}{2in}
    \footnotesize
      \begin{tabular}{c|l}
        Kernel & $\Phi$ Equations\\
        \midrule
            Linear
            	& $\Phi=X^T\Gamma X$ \\
            Squared
            	& $\Phi=X^T \mathcal{L}_{\Gamma}X$ \\  
            Polynomial
            	& $\Phi=X^T\Psi X$ 
             	\hspace{0.2cm}
            	,
            	\hspace{0.2cm}           	
            	$\Psi = \Gamma \odot K_{XW,p-1}$ \\ 
            Gaussian
            	& $\Phi=-X^T\mathcal{L}_{\Psi} X$ 
            	,
            	$\Psi = \Gamma \odot K_{XW}$\\
            Multiquadratic
            	& $\Phi=X^T\mathcal{L}_{\Psi} X$ 
            	,
            	$\Psi = \Gamma \odot K_{XW}^{(-1)}$\\ 
        \bottomrule       
      \end{tabular}
      \caption{Equations for $\Phi$s for the common kernels.}
      \label{table:phis}
\end{minipage}
\end{table}

\summary{A brief comment about $\Phi$}
\textbf{$\Phi$ for Common Kernels.} After deriving $\Phi/\Phi_0$ pairs for the most common kernels, we note several recurrent characteristics. First, $\Phi$ scales with the dimension $d$ instead of the size of the data $n$. Since $n \gg d$ is  common across many datasets, the eigendecomposition performed on $\Phi \in \mathcal{R}^{d \times d}$ can be significantly faster while requiring less memory. Second, following Eq.~(\ref{eq:def_of_pseudo_kernel}), they are highly efficient to compute since a vectorized formulation of $\Phi$ can be derived for each kernel as shown in Table~\ref{table:phis}; commonly, they reduce to a dot product between a Laplacian matrix $\mathcal{L}$ with the data matrices $X$. The occurrence of the Laplacian is particularly surprising since nowhere in Eq.~(\ref{eq:obj_1}) suggests this relationship. \chieh{Third, observe from Table 3 that $\Phi$ can be expressed as $X^T \Omega X$, where $\Omega$ is a positive semi-definite (PSD) matrix. Since the formulation of $\Phi$ without $\Omega$ is the covariance matrix $X^TX$, $\Omega$ scales the covariance matrix by incorporating both the kernel and the label information. In addition, by applying the Cholesky decomposition on $\Omega$ to rewrite $\Phi$ as $(X^TL)(L^TX)$, $L$ becomes a matrix that adjusts the data itself. Therefore, IKDR can be interpreted as applying PCA on the adjusted data $L^TX$ where the kernel and label information is included.}


\summary{What was the ISM guarantees and how we extend it.}
\textbf{Extending the ISM Theoretical Guarantees. }
The main theorem in \citet{wu2018iterative} proves that a fixed point $W^*$ of Algorithm~\ref{alg:ism} is a local maximum of Eq.~(\ref{eq:obj_1}) \textit{only if} the Gaussian kernel is used. Our work extends the theorem to a family of kernels which we refer to as the ISM family. Here, we supply the theoretical foundation for this claim by first providing the following definition.

\begin{definition}
    \label{def:ism_family}
    Given $\beta = a(x_i,x_j)^TW W^{T} b(x_i, x_j)$ with $a(x_i,x_j)$ and $b(x_i,x_j)$ as functions of $x_i$ and $x_j$, any twice differentiable kernel that can be written in terms of $f(\beta)$ while retaining its symmetric positive semi-definite property is an ISM kernel belonging to the ISM family with an associated $\Phi$ matrix defined as
    \begin{align}
    \Phi = \frac{1}{2} \sum_{i,j} \Gamma_{i,j} f'(\beta)A_{i,j}.
    \label{eq:def_of_pseudo_kernel}
    \end{align}    
    \chieh{where $A_{i,j} = b(x_i,x_j) a(x_i,x_j)^T + a(x_i,x_j) b(x_i,x_j)^T$}.
\end{definition}


\begin{table}[h]   
\begin{center}
    \begin{tabular}{l|c|c|c}
    Kernel Name & $f(\beta)$ & $a(x_i,x_j)$ & $b(x_i, x_j)$\\ \hline
    Linear & $\beta$ & $x_i$ & $x_j$ \\
    Squared & $\beta$ & $x_i-x_j$ & $x_i-x_j$ \\
    Polynomial & $(\beta + c)^p$ & $x_i$ & $x_j$ \\
    Gaussian &  $e^\frac{-\beta}{2\sigma^2}$ & $x_i-x_j$ & $x_i-x_j$ \\
    Multiquadratic & $\sqrt{\beta+c^2}$ & $x_i-x_j$ & $x_i-x_j$ \\
    \end{tabular}
    \caption{Converting common kernels to $f(\beta)$.}
    \label{table:kernels}
\end{center}
\end{table}

\summary{Why ISM theory was not obvious, and how definition 1 was the key to generalize ISM. We show that even a linear combination of kernels still fits this definition}
Since the equation for different kernels varies vastly, it is not clear how they can be reformulated into a single structure that simultaneously satisfies all ISM guarantees. Definition~\ref{def:ism_family} is the key realization that unites a set of kernels into a family. Under this definition, we proved later in Theorem~\ref{thm:stationary} that the Gaussian kernel within the original proof of ISM can be replaced by $f(\beta)$. Therefore, the ISM guarantees simultaneously extend to any kernel that satisfies Definition~\ref{def:ism_family}. As a result, a general $\Phi$ for \textit{any ISM kernel} can be derived as shown in Eq.~(\ref{eq:def_of_pseudo_kernel}). Moreover, note that the family of potential kernels is not limited to a finite set of known kernels, instead, it extends to any conic combinations of ISM kernels. We prove in Appendix~\ref{app:linear_combinations_of_kernels} the following proposition.

\begin{proposition}\label{thm:linear_combinations_of_kernels}
    Any conic combination of ISM kernels is still an ISM kernel.
\end{proposition}

\summary{The Properties of $\Phi$, which leads into Theorem}
\textbf{Properties of $\Phi$. }Since each ISM kernel is coupled with its own $\Phi$ matrix, $\Phi$s can conceptually replace kernels. Recall that the $V_{\text{max}}$ of $\Phi$ \textit{for any kernel in the ISM family} is the local maximum of Eq.~(\ref{eq:obj_1}). This central property is established in the following two theorems.




\begin{theorem}\label{thm:stationary}
    Given a full rank $\Phi$ with an eigengap as defined by Eq.~(\ref{eq:final_conclusion}) in Appendix~\ref{app:theorem_1_proof}, a fixed point $W^*$ of Algorithm~\ref{alg:ism} satisfies the 2nd Order Necessary Conditions (Theorem 12.5~\cite{wright1999numerical}) for Eq. (\ref{eq:obj_1}) using any ISM kernel.
\end{theorem}
\begin{theorem}\label{thm:convergence}
  A sequence of subspaces $\{W_k W _ k^T\}_{k\in \mathbb{N}}$ generated by Algorithm~\ref{alg:ism} contains a convergent subsequence. 
 \end{theorem}

\summary{A quick description of the new generalized ISM proof, and how we differ.}
Since the entire ISM proof along with its convergence guarantee is required to be revised and generalized under Definition~\ref{def:ism_family}, we leave the detail to Appendix~\ref{app:theorem_1_proof} and \ref{app:convergence} while presenting here only the main conclusions. Functionally, our proof is separated into two lemmas to establish $\Phi$ as a kernel surrogate. Lemma 1 concludes that given any $\Phi$ of an ISM kernel, the gradient of the Lagrangian for Eq.~(\ref{eq:obj_1}) is equivalent to $-\Phi W - W \Lambda$. Therefore, when the gradient is set to 0, the eigenvectors of $\Phi$ is equivalent to the stationary point of Eq.~(\ref{eq:obj_1}). For Lemma 2, given $\bar{\Lambda}$ as the eigenvalues associated with the eigenvectors \textit{not chosen} and $\mathcal{C}$ as constant, it concludes that the 2nd order necessary condition is satisfied when $(\min_i  \bar{\Lambda}_i - \max_j \Lambda_j) \ge \mathcal{C}.$  This inequality indicates the necessity for the smallest eigenvalue among the un-chosen eigenvectors to be greater than the maximum eigenvalue of the chosen by at least $\mathcal{C}$.  Therefore, given the choice of $q$ eigenvectors, the $q$ smallest eigenvalues will maximize the gap. This is equivalent to finding the most dominant eigenvectors of $\Phi$. Putting both lemmas together, we conclude that the most dominant eigenvectors of any $\Phi$ within the ISM family is the solution to Eq.~(\ref{eq:obj_1}).

\summary{How ISM is initialized}
\textbf{Initializing $W$ with $\Phi_0$.} After generalizing ISM, different $\Phi$s may or may not be a function of $W$. When $\Phi$ is not a function of $W$, the $V_{\text{max}}$ of $\Phi$ is immediately the solution. However, if $\Phi$ is a function of $W$, $\Phi$ iteratively updates from the previous $W$. This process is initialized using a $\Phi_0$ that is independent of $W$. To obtain $\Phi_0$, ISM approximates the Gaussian kernel up to the 2nd order of the Taylor series around $\beta=0$ and discovers that the approximation of $\Phi$ is independent of $W$. Our work leverages Definition~\ref{def:ism_family} and proves that a common formulation for $\Phi_0$ is possible. We formalize our finding in the following theorem and provided the proof in Appendix~\ref{app:approx_phi}.

\begin{theorem}
    \label{prop:initial_pseudo_kernel}
    For any kernel within the ISM family, a $\Phi$ independent of $W$ can be approximated with
    \begin{align} 
    \Phi \approx \sign(\nabla_{\beta} f(0)) \sum_{i,j} \Gamma_{i,j} A_{i,j}.
    \end{align} 
\end{theorem}

\summary{Extending ISM to Conic Combination of Kernels.}
\textbf{Extending ISM to Conic Combination of Kernels.} 
The two lemmas of Theorem~\ref{thm:stationary} highlights the conceptual convenience of working with $\Phi$ in place of kernels. This conceptual replacement extends even to conic combinations of ISM kernels. As a corollary to Theorem~\ref{thm:stationary}, we discovered that when a kernel is constructed through a conic combination of ISM kernels, it also has an associated $\Phi$ matrix. Remarkably, it is equivalent to the conic combination of $\Phi$s from individual kernels using the same coefficients. Formally, we propose the following corollary with its proof in Appendix~\ref{app:combined_kernels}. 

\begin{corollary}\label{corollary:combine_kernels}
    The $\Phi$ matrix associated with a conic combination of kernels is the conic combination of $\Phi$s associated with each individual kernel.
\end{corollary}

\textbf{Complexity analysis. } 
    Let $t$ be the number of iterations required for convergence and $n \gg d$, ISM's time complexity is dominated by the dot product between $\mathcal{L} \in \mathcal{R}^{n \times n}$ and $X \in \mathcal{R}^{n \times d}$. Together ISM has a time complexity of $O(n^2dt)$; a significant improvement from DG $O(n^2dq^2t)$, or SM at $O(n^2dqt)$. ISM is also faster since $t$ is significantly smaller. While $t$ ranges from hundreds to thousands for competing algorithms, ISM normally converges at $t<5$.  In terms of memory, ISM faces similar challenges as all kernel methods where the memory complexity is upper bounded at $O(n^2)$.

\section{Experiments}

 \textbf{Datasets.}
    The experiment includes 5 real datasets of commonly encountered data types. Wine~\cite{Dua:2017} consists of continuous data while the Cancer dataset ~\cite{breastcancer} features are discrete. The Face dataset ~\cite{bay2000uci} is a standard dataset used for alternative clustering; it includes images of 20 people in various poses. The MNIST~\cite{deng2012mnist} dataset includes images of handwritten characters. The Face and the MNIST datasets are chosen to highlight ISM's ability to handle images. The Flower image by Alain Nicolas \cite{particleoft} is another dataset chosen for alternative clustering where we seek alternative ways to perform image segmentation. For more in-depth details on each dataset, see Appendix~\ref{app:data_detail}.
    
 \textbf{Experimental Setup.}
    We showcase ISM's efficacy on three different learning paradigms, i.e., supervised dimension reduction\cite{masaeli2010transformation}, unsupervised clustering \cite{niu2011dimensionality}, and semi-supervised alternative clustering \cite{chang2017clustering}. As an optimization technique, we compare ISM in Table~\ref{table:result_table} against competing state-of-the-art manifold optimization algorithms: Dimension Growth \textbf{(DG)} \cite{niu2014iterative}, the Stiefel Manifold approach \textbf{(SM)} \cite{wen2013feasible}, and the Grassmann Manifold \textbf{(GM)} \cite{boumal2011rtrmc,manopt}. To emphasize ISM family of kernels, the supervised and unsupervised results using several less conventional kernels are included in Table~\ref{table:other_kernels_too}. Within this table, we also investigate using conic combination of $\Phi$s by combining the Gaussian and the polynomial kernels with center alignment~\cite{cortes2012algorithms}. Since center alignment is specific to supervised cases, this is not repeated for the unsupervised case.

    For \textit{supervised} dimension reduction, we perform SVM on $XW$ using 10-fold cross validation. \chieh{For each of the 10-fold experiments, we trained $W$ and the SVM classifier only on the training set while reporting the result only on the test set, i.e., the test set was never used during the training.} We repeat this process for each fold of cross-validation. From the 10-fold results in Table~\ref{table:result_table}, we record the mean and the standard deviation of the run-time, cost, and accuracy. We investigate the scalability in Figure~\ref{fig:scalability} by comparing the change in run-time as we increment the sample size. For \textit{unsupervised} dimension reduction, we perform spectral clustering on $XW$ after learning $W$ where we record the run-time, cost, and NMI. For \textit{alternative clustering}, we highlight the ISM family of kernels by reproducing the original ISM results (generated with Gaussian kernel) using the polynomial kernel. On the Flower image, each sample is a $\mathcal{R}^3$ vector. We supply the original image segmentation result as semi-supervised labels and learn an alternative way to segment the image. The original segmentation and the alternative segmentation are shown in Figure\ref{fig:alt_cluster}. For the Face dataset, each sample is a vector vectorized from a grayscaled image of individuals. We provide the identity of individuals as the original clustering label and search for an alternative way to cluster the data.

 \textbf{Evaluation Metric.}
    In the supervised case, the test classification accuracy from the 10-fold cross validation is recorded along with the cost and run-time. The time is broken down into days (d), hours (h), minutes (m), and seconds (s). The best results are bold for each experiment. In the unsupervised case, we report the Normalized Mutual Information (NMI) \cite{strehl2002cluster} to compare the clustering labels against the ground truth. For detail on how NMI is computed, see Appendix~\ref{app:NMI_computation}.
    

 \textbf{Experiment Settings. }
    The median of the pair-wise Euclidean distance is used as $\sigma$ for all experiments using the Gaussian kernel. Degree of 3 is used for all polynomial kernels. The dimension of subspace $q$ is set to the number of classes/clusters. The convergence threshold $\delta$ is set to 0.01. All competing algorithms use their default initialization. All datasets are centered to 0 and scaled to a standard deviation of 1. All sources are written in Python using Numpy and Sklearn \cite{numpy,sklearn_api}. All experiments were conducted on Dual Intel Xeon E5-2680 v2 @ 2.80GHz, with 20 total cores. Due to limited computational resources, each run is limited to 3 days.

\textbf{Complexity Analysis of Competing Methods. }    
    The run-time as a function of linearly increasing sample size is shown for the polynomial kernel in Figure~\ref{fig:scalability}. Since the complexity analysis for ISM suggests a relationship of $O(n^2)$ with respect to the sample size, $log_2(.)$ is used for the $Y$-axis. As expected, the ISM's linear run-time growth in Figure~\ref{fig:scalability} supports our analysis of $O(n^2)$ relationship. The plot for competing algorithms reported a similar linear relationship with comparable slopes. This indicates that the difference in speed is not a function of the data size, but other factors such as $q$ and $t$. Using \textbf{DG}'s complexity of $O(n^2dq^2t)$ as an example, it normally converges when $t$ is in the ranges of thousands. Since $q=20$ was used in the figure, the significant speed improvement from ISM can be derived from the $q^2t$ factor since ISM generally converges at $t$ less than 5.

\summary{ISM is a faster optimization technique that produces a lower cost. }
\textbf{Results. }
    Comparing against other optimization algorithms in  Table~\ref{table:result_table}, the results confirm ISM as a significantly faster algorithm while consistently achieving a lower cost. This disparity is especially prominent when the data dimension $q$ is higher. We highlight that for the Face dataset on the Gaussian kernel, it took DG 1.92 days, while ISM finished within 0.99 seconds: a $10^5$-fold speed difference. To further confirm these advantages, the same experiment is repeated using the \textit{polynomial kernel} where similar results can be observed. Besides the execution time and cost, the classification accuracy across 5 datasets never falls below 95\% in the supervised setting. The same datasets and techniques are repeated in an unsupervised clustering problem. While the clustering quality is comparable across the datasets, ISM clearly produces the lowest cost with the fastest execution time.

\summary{Table~\ref{table:other_kernels_too} shows how we generalize ISM to other kernels.}
    Table~\ref{table:other_kernels_too} focuses on the generalization of ISM to a family of kernels. Since Table~\ref{table:result_table} already supplied results from the Gaussian and polynomial kernel, we feature 4 more kernels to support the claim. As kernel methods treat kernels as interchangeable components of the algorithm, ISM achieves a similar effect by replacing the $\Phi$ matrix. As evidenced from the table, similar accuracy and time can be achieved with this replacement without affecting the rest of the algorithm. In many cases, the multiquadratic kernel outperforms even the Gaussian and the polynomial kernel. In a similar spirit, we repeated the same experiments in the unsupervised case and received further confirmation. 
   
\summary{We have even shown that a combination $\Phi$s can be used in place of kernels.}
    To support Corollary~\ref{corollary:combine_kernels}, results using a Gaussian + polynomial (G+P) kernel is also supplied in Table~\ref{table:other_kernels_too}. It is not surprising that a combination of $\Phi$s is the best performing kernel. Since the union of the two kernels covers a larger feature space, the expressiveness is also greater. This result supports the claim that a conic combination of $\Phi$s can replace the same combination of kernels for Eq.~(\ref{eq:obj_1}).

\summary{For semi-supervised, we reproduced the original alternative clustering results using the polynomial kernels instead of the gaussian kernel.}
    To study the generalized ISM on a (semi-supervised) alternative clustering problem, we use it to recreate the results from the original paper on alternative clustering. We emphasize that our results differ in the choice of using the polynomial kernel instead of the Gaussian. From the Flower experiment, it is visually clear that the original image segmentation of 2 clusters (separated by black and white) is completely different from the alternative segmentation. For the Face data, the original clusters were grouped by the identity of the individuals while the algorithm produced 4 alternative clusters. By averaging the images of each alternative cluster, the new clustering pattern can be visually seen in Figure~\ref{fig:alt_cluster}; the samples are alternatively clustered by the pose.

    \begin{figure}[h]
      \begin{subfigure}[b]{0.5\textwidth}
        \includegraphics[width=\textwidth,height=6cm]{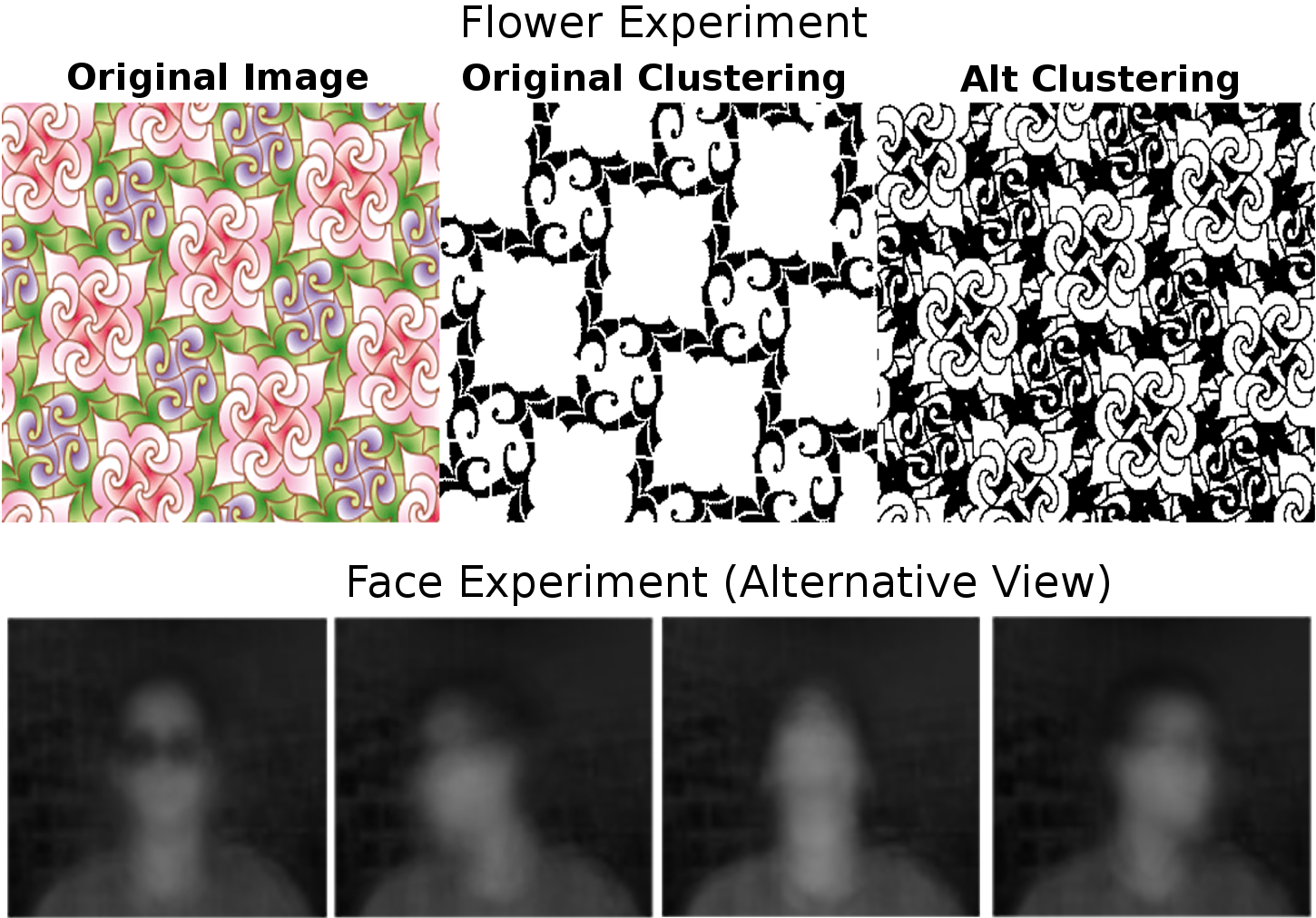}
        \caption{Reproducing results from the original ISM paper using polynomial kernels.}
        \label{fig:alt_cluster}
      \end{subfigure}
      \begin{subfigure}[b]{0.5\textwidth}
        \includegraphics[width=\textwidth,height=6cm]{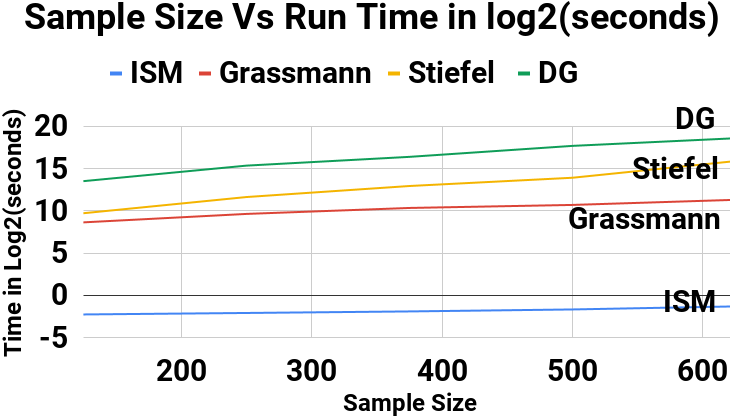}
        \caption{Log2 run-time as a function of increasing samples.}
        \label{fig:scalability}
      \end{subfigure}
    \end{figure}

    \begin{table}[h]
        \tiny
        \centering
        \setlength{\tabcolsep}{3.0pt}
        \renewcommand{\arraystretch}{1.2}
      \begin{tabular}{|cc|c|c|c|c|c|c|c|c|c|}
        \hline
          \multicolumn{2}{|c|}{\textbf{Supervised}} &
          \multicolumn{4}{|c|}{\textbf{Gaussian}} &
          \multicolumn{4}{|c|}{\textbf{polynomial}} \\
        \cline{3-10}
        & & ISM & DG & SM & GM & 
        ISM & DG & SM & GM \\
        \Xhline{2\arrayrulewidth}
        \parbox[t]{1mm}{\multirow{3}{*}{\rotatebox[origin=c]{90}{\textbf{Wine}}}}&
            \textbf{Time} & 
                \textbf{0.02s} $\pm$ \textbf{0.01s} & 
                7.9s $\pm$ 2.9s & 
                1.7s $\pm$ 0.7s & 
                16.8m $\pm$ 3.4s & 
                \textbf{0.02s} $\pm$ \textbf{0.0s} & 
                13.2s $\pm$ 6.2s &
                14.77s $\pm$ 0.6s &
                16.82m $\pm$ 3.6s  \\
        &
        \textbf{Cost} & 
            \textbf{-1311} $\pm$ \textbf{26}& 
            -1201 $\pm$ 25& 
            -1310 $\pm$ 26& 
            -1307 $\pm$ 25 & 
            \textbf{-114608} $\pm$ \textbf{1752} & 
            -112440 $\pm$ 1719 & 
            -111339 $\pm$ 1652& 
            -108892 $\pm$ 1590 \\
        &
        \textbf{Accuracy} & 
            \textbf{95.0\%} $\pm$ \textbf{5\%} & 
            93.2\% $\pm$ 5.5\% & 
            \textbf{95\%} $\pm$ \textbf{4.2\%} & 
            \textbf{95\%} $\pm$ \textbf{6\%} & 
            \textbf{97.2\%} $\pm$ \textbf{3.7\%} & 
            93.8\% $\pm$ 3.9\% & 
            96.6\% $\pm$ 3.7\% & 
            96.6\% $\pm$ 2.7\% \\
         \hline
        \parbox[t]{1mm}{\multirow{3}{*}{\rotatebox[origin=c]{90}{\textbf{Cancer}}}}&
            \textbf{Time} & 
                \textbf{0.08s} $\pm$ \textbf{0.0s} & 
                4.5m $\pm$ 103s & 
                17s $\pm$ 12s & 
                17.8m $\pm$ 80s & 
                \textbf{0.13s} $\pm$ \textbf{0.0s} & 
                4m $\pm$ 1.2m &
                3.3m $\pm$ 3s &
                17.5m $\pm$ 1.1m \\
        &
        \textbf{Cost} & 
            \textbf{-32249} $\pm$ \textbf{338} & 
            -30302 $\pm$ 2297 & 
            -31996 $\pm$ 499 & 
            -30998 $\pm$ 560& 
            \textbf{-1894} $\pm$ \textbf{47} & 
            -1882 $\pm$ 47 & 
            -1737 $\pm$ 84 & 
            -1690 $\pm$ 108 \\                  
        &
        \textbf{Accuracy} & 
            97.3\%$\pm$ 0.3\% & 
            97.3\%$\pm$ 0.3\% & 
            97.3\%$\pm$ 0.2\% & 
            \textbf{97.4\%}$\pm$ \textbf{0.4\%} & 
            \textbf{97.4\%}$\pm$ \textbf{0.3\%} & 
            97.3\% $\pm$ 0.3\% & 
            \textbf{97.4\%} $\pm$ \textbf{0.3\%} & 
            97.3\% $\pm$ 0.3\% \\
        \hline
        \parbox[t]{1mm}{\multirow{3}{*}{\rotatebox[origin=c]{90}{\textbf{Face}}}}&
            \textbf{Time} & 
                \textbf{0.99s} $\pm$ \textbf{0.1s} & 
                1.92d $\pm$ 11h & 
                10s $\pm$ 5s & 
                22.7m $\pm$ 18s & 
                \textbf{0.7s} $\pm$ \textbf{0.03s} & 
                2.1d $\pm$ 13.9h &
                5.0m $\pm$ 5.7s &
                21.5m $\pm$ 9.8s\\
        &
        \textbf{Cost} & 
            \textbf{-3754} $\pm$ \textbf{31} & 
            -3431 $\pm$ 32 & 
            -3749 $\pm$ 33 & 
            -771 $\pm$ 28 & 
            \textbf{-82407} $\pm$ \textbf{1670} & 
            -78845 $\pm$ 1503 & 
            -37907 $\pm$ 15958 & 
            -3257 $\pm$ 517\\                  
        &
        \textbf{Accuracy} & 
            \textbf{100\%} $\pm$ \textbf{0\%} & 
            \textbf{100\%} $\pm$ \textbf{0\%} & 
            \textbf{100\%} $\pm$ \textbf{0\%} & 
            99.2\% $\pm$ 0.2\% & 
            \textbf{100\%} $\pm$ \textbf{0\%} & 
            \textbf{100\%} $\pm$ \textbf{0\%} & 
            \textbf{100\%} $\pm$ \textbf{0\%} & 
            99.8\% $\pm$ 0.2\% \\
         \hline
         \parbox[t]{1mm}{\multirow{3}{*}{\rotatebox[origin=c]{90}{\textbf{MNIST}}}}&
            \textbf{Time} & 
                \textbf{13.8s $\pm$ 2.3s} &
                > 3d &
                2.5m $\pm$ 1.0s &
                > 3d &
                \textbf{12.1s $\pm$ 1.4s} & 
                > 3d &
                2.1m $\pm$ 3s &
                > 3d \\
        &
        \textbf{Cost} & 
            \textbf{-639 $\pm$ 2.3} &
            N/A &
            -621 $\pm$ 5.1 &
            N/A &
            \textbf{-639 $\pm$ 2} & 
            N/A &
            -620 $\pm$ 5.1 &
            N/A \\
        &
        \textbf{Accuracy} & 
            \textbf{99\%} $\pm$ \textbf{0\%} &
            N/A & 
            98.5\% $\pm$ 0.4\% & 
            N/A & 
            \textbf{99\%} $\pm$ \textbf{0\%} &
            N/A & 
            \textbf{99\%} $\pm$ \textbf{0\%} &
            N/A \\
        \Xhline{2\arrayrulewidth}
         \multicolumn{2}{|c|}{\textbf{Unsupervised}} \\
        \Xhline{2\arrayrulewidth}
         \parbox[t]{1mm}{\multirow{3}{*}{\rotatebox[origin=c]{90}{\textbf{Wine}}}}&
            \textbf{Time} & 
                \textbf{0.01s} & 
                9.9s & 
                0.6s & 
                16.7m & 
                \textbf{0.02s} & 
                14.4s &
                2.9s &
                33.5m \\
        &
        \textbf{Cost} & 
            \textbf{-27.4} & 
            -25.2 & 
            -27.3 & 
            -27.3 & 
            \textbf{-1600} & 
            -1582 & 
            -1598 & 
            -1496 \\
        &
        \textbf{NMI} & 
            \textbf{0.86} & 
            \textbf{0.86} & 
            \textbf{0.86} & 
            \textbf{0.86} & 
            \textbf{0.84} & 
            \textbf{0.84} & 
            \textbf{0.84} & 
            0.83 \\
        \hline
          \parbox[t]{1mm}{\multirow{3}{*}{\rotatebox[origin=c]{90}{\textbf{Cancer}}}}&
            \textbf{Time} & 
                \textbf{0.57s} & 
                4.3m & 
                3.9s & 
                44m & 
                \textbf{0.5s} & 
                8.0m &
                8.8m &
                41m \\
        &
        \textbf{Cost} & 
            \textbf{-243} & 
            -133 & 
            -146 & 
            -142 & 
            \textbf{-15804} & 
            -14094& 
            -15749& 
            -11985\\                  
        &
        \textbf{NMI} & 
            \textbf{0.8} & 
            0.79 & 
            \textbf{0.8} & 
            0.79 & 
            0.79 & 
            \textbf{0.80} & 
            0.79 & 
            \textbf{0.80} \\
         \hline
          \parbox[t]{1mm}{\multirow{3}{*}{\rotatebox[origin=c]{90}{\textbf{Face}}}}&
            \textbf{Time} & 
                \textbf{0.3s} & 
                1.3d & 
                5.3s & 
                55.9m & 
                \textbf{1.0s} & 
                > 3d &
                22m &
                1.6d\\
        &
        \textbf{Cost} & 
            \textbf{-169.3} & 
            -167.7 & 
            -168.9 & 
            -37 & 
            \textbf{-368} & 
            NA & 
            -348 & 
            -321 \\
        &
        \textbf{NMI} & 
            0.94 & 
            \textbf{0.95} & 
            0.93 & 
            0.89 & 
            \textbf{0.94} & 
            N/A & 
            0.89 & 
            0.89 \\
          \hline
         \parbox[t]{1mm}{\multirow{3}{*}{\rotatebox[origin=c]{90}{\textbf{MNIST}}}}&
            \textbf{Time} & 
                \textbf{1.8h} & 
                > 3d &
                1.3d &
                > 3d &
                \textbf{8.3m} &
                > 3d &
                0.9d &
                > 3d \\
        &
        \textbf{Cost} & 
            \textbf{-2105} &
            N/A &
            -2001 &
            N/A &
            \textbf{-51358} &
            N/A &
            -51129 &
            N/A \\
        &
        \textbf{NMI} & 
            \textbf{0.47} &
            N/A &
            0.46 &
            N/A &
            \textbf{0.32} &
            N/A &
            \textbf{0.32} &
            N/A \\
         \Xhline{2\arrayrulewidth}
      \end{tabular}
      \caption{Run-time, cost, and objective performance are recorded under supervised/unsupervised objectives. ISM is significantly faster compared to other optimization techniques while achieving lower objective cost.}
    \label{table:result_table}
    \end{table}

    \begin{table}[h]
        \tiny
        \centering
        \setlength{\tabcolsep}{4.0pt}
        \renewcommand{\arraystretch}{1.5}
      \begin{tabular}{|cc|c|c|c|c|c|c|c|c|c|c|}
        \hline
       & &
       \multicolumn{4}{|c|}{\textbf{Supervised}} &
       \multicolumn{4}{|c|}{\textbf{Unsupervised}} \\
        \hline
        & & \textbf{Linear} & \textbf{Squared} 
        & \textbf{Multiquad} & \textbf{G+P} 
        & & \textbf{Linear} & \textbf{Squared} & \textbf{Multiquad}  \\
        \Xhline{2\arrayrulewidth}
        \parbox[t]{1mm}{\multirow{2}{*}{\rotatebox[origin=c]{90}{\textbf{Wine}}}}&
            \textbf{Time} & 
                \textbf{0.003s $\pm$ 0s} & 
                0.01s $\pm$ 0s & 
                0.02s $\pm$ 0.01s & 
                0.007s $\pm$ 0s  &
            \textbf{Time} & 
                \textbf{0.02s} &
                0.04s &
                0.06s \\
        &
        \textbf{Accuracy} & 
            97.2\% $\pm$ 2.8\% & 
            96.6\% $\pm$ 3.7\% & 
            97.2\% $\pm$ 3.7\% & 
            \textbf{98.3\% $\pm$ 2.6\%} &
        \textbf{NMI} &
            0.85 &
            0.85 &
            \textbf{0.88} \\
         \hline
        \parbox[t]{1mm}{\multirow{2}{*}{\rotatebox[origin=c]{90}{\textbf{Cancer}}}}&
            \textbf{Time} & 
                \textbf{0.02s $\pm$ 0.002s} & 
                0.09s $\pm$ 0.02s & 
                0.15s $\pm$ 0.01s & 
                0.06s $\pm$ 0.004s &
            \textbf{Time} & 
                \textbf{0.23s} &
                0.5s &
                0.56s \\
        &
        \textbf{Accuracy} & 
            97.2\% $\pm$ 0.3\% & 
            97.3\% $\pm$ 0.04\% & 
            \textbf{97.4\% $\pm$ 0.003\%} & 
            \textbf{97.4\% $\pm$ 0.003\%} &
        \textbf{NMI} & 
            0.80 &
            0.79 &
            \textbf{0.84} \\
        
        \hline
          \parbox[t]{1mm}{\multirow{2}{*}{\rotatebox[origin=c]{90}{\textbf{Face}}}}&
            \textbf{Time} & 
                \textbf{0.2s $\pm$ 0.2s} & 
                0.3s $\pm$ 0.2s & 
                0.3s $\pm$ 0.2s & 
                0.5s $\pm$ 0.03s &
            \textbf{Time} & 
                \textbf{0.68s} & 
                0.92s & 
                3.7s \\ 
        &
        \textbf{Accuracy} & 
            97.3\% $\pm$ 0.3\% & 
            97.1\% $\pm$ 0.4\% & 
            97.3\% $\pm$ 0.4\% & 
            \textbf{100\% $\pm$ 0\%} &
        \textbf{NMI} & 
            0.93 &
            \textbf{0.95} &
            0.92 \\
         \hline
          \parbox[t]{1mm}{\multirow{2}{*}{\rotatebox[origin=c]{90}{\textbf{MNIST}}}}&
            \textbf{Time} & 
                \textbf{6.4s $\pm$ 0.4s} & 
                17.4s $\pm$ 0.4s & 
                10.6m $\pm$ 1.9m & 
                17.6s $\pm$ 2.5s &
            \textbf{Time} & 
            \textbf{3.1m} &
            4.7m &
            52m\\
        &
        \textbf{Accuracy} & 
            99.1\% $\pm$ 0.1\% & 
            \textbf{99.3\% $\pm$ 0.2\%} & 
            99.1\% $\pm$ 0.1\% & 
            \textbf{99.3\% $\pm$ 0.2\%} &
        \textbf{NMI} & 
            \textbf{0.54} &
            \textbf{0.54} &
            \textbf{0.54} \\
         \Xhline{2\arrayrulewidth}
      \end{tabular}
      \caption{Run-time and objective performance are recorded across several kernels within the ISM family. It confirms the usage of $\Phi$ or linear combination of $\Phi$ in place of kernels.}
    \label{table:other_kernels_too}
    \end{table}

\clearpage

\summary{The experiments on 3 different paradigms is supporting evidence ISM can be more widely used.}
    By applying ISM to three different learning paradigms, we showcase ISM as an extremely fast optimization algorithm that can solve a wide range of IKDR problems, thereby drawing a deeper connection between these domains.
     Hence, the impact of generalizing ISM to other kernels is also conveniently translated to these applications.

\section{Conclusion}
    \chieh{We have extended the theoretical guarantees of ISM to a family of kernels beyond the Gaussian kernel via the discovery of the $\Phi$ matrix. Our theoretical analysis proves that the family of ISM kernels extend even to conic combinations of ISM kernels. With this extension, ISM becomes an efficient solution for a wide range of supervised, unsupervised and semi-supervised applications. Our experimental results confirm the efficiency of the algorithm while showcasing its wide impact across many domains.}

\textbf{Acknowledgments}

We would like to acknowledge support for this project from NSF grant IIS-1546428. We would also like to thank Zulqarnain Khan for his insightful discussions.

\bibliography{reference}

\begin{thebibliography}{44}
\providecommand{\natexlab}[1]{#1}
\providecommand{\url}[1]{\texttt{#1}}
\expandafter\ifx\csname urlstyle\endcsname\relax
  \providecommand{\doi}[1]{doi: #1}\else
  \providecommand{\doi}{doi: \begingroup \urlstyle{rm}\Url}\fi

\bibitem[Niu et~al.(2011)Niu, Dy, and Jordan]{niu2011dimensionality}
Donglin Niu, Jennifer Dy, and Michael Jordan.
\newblock Dimensionality reduction for spectral clustering.
\newblock In \emph{Proceedings of the Fourteenth International Conference on
  Artificial Intelligence and Statistics}, pages 552--560, 2011.

\bibitem[Suzuki and Sugiyama(2010)]{suzuki2010sufficient}
Taiji Suzuki and Masashi Sugiyama.
\newblock Sufficient dimension reduction via squared-loss mutual information
  estimation.
\newblock In \emph{Proceedings of the Thirteenth International Conference on
  Artificial Intelligence and Statistics}, pages 804--811, 2010.

\bibitem[Jolliffe(2011)]{jolliffe2011principal}
Ian Jolliffe.
\newblock \emph{Principal component analysis}.
\newblock Springer, 2011.

\bibitem[Roweis and Saul(2000)]{roweis2000nonlinear}
Sam~T Roweis and Lawrence~K Saul.
\newblock Nonlinear dimensionality reduction by locally linear embedding.
\newblock \emph{science}, 290\penalty0 (5500):\penalty0 2323--2326, 2000.

\bibitem[Wickelmaier(2003)]{wickelmaier2003introduction}
Florian Wickelmaier.
\newblock An introduction to mds.
\newblock \emph{Sound Quality Research Unit, Aalborg University, Denmark},
  46\penalty0 (5):\penalty0 1--26, 2003.

\bibitem[Tenenbaum et~al.(2000)Tenenbaum, De~Silva, and
  Langford]{tenenbaum2000global}
Joshua~B Tenenbaum, Vin De~Silva, and John~C Langford.
\newblock A global geometric framework for nonlinear dimensionality reduction.
\newblock \emph{science}, 290\penalty0 (5500):\penalty0 2319--2323, 2000.

\bibitem[Kambhatla and Leen(1997)]{kambhatla1997dimension}
Nandakishore Kambhatla and Todd~K Leen.
\newblock Dimension reduction by local principal component analysis.
\newblock \emph{Neural computation}, 9\penalty0 (7):\penalty0 1493--1516, 1997.

\bibitem[Fukumizu et~al.(2009)Fukumizu, Bach, Jordan,
  et~al.]{fukumizu2009kernel}
Kenji Fukumizu, Francis~R Bach, Michael~I Jordan, et~al.
\newblock Kernel dimension reduction in regression.
\newblock \emph{The Annals of Statistics}, 37\penalty0 (4):\penalty0
  1871--1905, 2009.

\bibitem[Song et~al.(2008)Song, Gretton, Borgwardt, and Smola]{song2008colored}
Le~Song, Arthur Gretton, Karsten Borgwardt, and Alex~J Smola.
\newblock Colored maximum variance unfolding.
\newblock In \emph{Advances in neural information processing systems}, pages
  1385--1392, 2008.

\bibitem[Song et~al.(2007)Song, Smola, Gretton, Borgwardt, and
  Bedo]{song2007supervised}
Le~Song, Alex Smola, Arthur Gretton, Karsten~M Borgwardt, and Justin Bedo.
\newblock Supervised feature selection via dependence estimation.
\newblock In \emph{Proceedings of the 24th international conference on Machine
  learning}, pages 823--830. ACM, 2007.

\bibitem[Vladymyrov and Carreira-Perpin{\'a}n(2013)]{vladymyrov2013locally}
Max Vladymyrov and Miguel~{\'A} Carreira-Perpin{\'a}n.
\newblock Locally linear landmarks for large-scale manifold learning.
\newblock In \emph{Joint European Conference on Machine Learning and Knowledge
  Discovery in Databases}, pages 256--271. Springer, 2013.

\bibitem[Song et~al.(2012)Song, Smola, Gretton, Bedo, and
  Borgwardt]{song2012feature}
Le~Song, Alex Smola, Arthur Gretton, Justin Bedo, and Karsten Borgwardt.
\newblock Feature selection via dependence maximization.
\newblock \emph{Journal of Machine Learning Research}, 13\penalty0
  (May):\penalty0 1393--1434, 2012.

\bibitem[Belkin and Niyogi(2003)]{belkin2003laplacian}
Mikhail Belkin and Partha Niyogi.
\newblock Laplacian eigenmaps for dimensionality reduction and data
  representation.
\newblock \emph{Neural computation}, 15\penalty0 (6):\penalty0 1373--1396,
  2003.

\bibitem[Elhamifar and Vidal(2013)]{elhamifar2013sparse}
Ehsan Elhamifar and Rene Vidal.
\newblock Sparse subspace clustering: Algorithm, theory, and applications.
\newblock \emph{IEEE transactions on pattern analysis and machine
  intelligence}, 35\penalty0 (11):\penalty0 2765--2781, 2013.

\bibitem[Sch{\"o}lkopf et~al.(1998)Sch{\"o}lkopf, Smola, and
  M{\"u}ller]{scholkopf1998nonlinear}
Bernhard Sch{\"o}lkopf, Alexander Smola, and Klaus-Robert M{\"u}ller.
\newblock Nonlinear component analysis as a kernel eigenvalue problem.
\newblock \emph{Neural computation}, 10\penalty0 (5):\penalty0 1299--1319,
  1998.

\bibitem[Barshan et~al.(2011)Barshan, Ghodsi, Azimifar, and
  Jahromi]{barshan2011supervised}
Elnaz Barshan, Ali Ghodsi, Zohreh Azimifar, and Mansoor~Zolghadri Jahromi.
\newblock Supervised principal component analysis: Visualization,
  classification and regression on subspaces and submanifolds.
\newblock \emph{Pattern Recognition}, 44\penalty0 (7):\penalty0 1357--1371,
  2011.

\bibitem[Sch{\"o}lkopf et~al.(1997)Sch{\"o}lkopf, Smola, and
  M{\"u}ller]{scholkopf1997kernel}
Bernhard Sch{\"o}lkopf, Alexander Smola, and Klaus-Robert M{\"u}ller.
\newblock Kernel principal component analysis.
\newblock In \emph{International conference on artificial neural networks},
  pages 583--588. Springer, 1997.

\bibitem[Gretton et~al.(2005)Gretton, Bousquet, Smola, and
  Sch{\"o}lkopf]{gretton2005measuring}
Arthur Gretton, Olivier Bousquet, Alex Smola, and Bernhard Sch{\"o}lkopf.
\newblock Measuring statistical dependence with hilbert-schmidt norms.
\newblock In \emph{International conference on algorithmic learning theory},
  pages 63--77. Springer, 2005.

\bibitem[Wu et~al.(2018)Wu, Ioannidis, Sznaier, Li, Kaeli, and
  Dy]{wu2018iterative}
Chieh Wu, Stratis Ioannidis, Mario Sznaier, Xiangyu Li, David Kaeli, and
  Jennifer Dy.
\newblock Iterative spectral method for alternative clustering.
\newblock In \emph{International Conference on Artificial Intelligence and
  Statistics}, pages 115--123, 2018.

\bibitem[Masaeli et~al.(2010)Masaeli, Dy, and Fung]{masaeli2010transformation}
Mahdokht Masaeli, Jennifer~G Dy, and Glenn~M Fung.
\newblock From transformation-based dimensionality reduction to feature
  selection.
\newblock In \emph{Proceedings of the 27th International Conference on Machine
  Learning (ICML-10)}, pages 751--758, 2010.

\bibitem[Gangeh et~al.(2016)Gangeh, Bedawi, Ghodsi, and Karray]{gangeh2016semi}
Mehrdad~J Gangeh, Safaa~MA Bedawi, Ali Ghodsi, and Fakhri Karray.
\newblock Semi-supervised dictionary learning based on hilbert-schmidt
  independence criterion.
\newblock In \emph{International Conference Image Analysis and Recognition},
  pages 12--19. Springer, 2016.

\bibitem[Chang et~al.(2017)Chang, Chen, Cho, Castaidi, Silverman, and
  Dy]{chang2017clustering}
Yale Chang, Junxiang Chen, Michael~H Cho, Peter~J Castaidi, Edwin~K Silverman,
  and Jennifer~G Dy.
\newblock Clustering with domain-specific usefulness scores.
\newblock In \emph{Proceedings of the 2017 SIAM International Conference on
  Data Mining}, pages 207--215. SIAM, 2017.

\bibitem[Niu et~al.(2010)Niu, Dy, and Jordan]{niu2010multiple}
Donglin Niu, Jennifer~G Dy, and Michael~I Jordan.
\newblock Multiple non-redundant spectral clustering views.
\newblock In \emph{Proceedings of the 27th international conference on machine
  learning (ICML-10)}, pages 831--838, 2010.

\bibitem[James(1976)]{james1976topology}
Ioan~Mackenzie James.
\newblock \emph{The topology of Stiefel manifolds}, volume~24.
\newblock Cambridge University Press, 1976.

\bibitem[Nishimori and Akaho(2005)]{nishimori2005learning}
Yasunori Nishimori and Shotaro Akaho.
\newblock Learning algorithms utilizing quasi-geodesic flows on the stiefel
  manifold.
\newblock \emph{Neurocomputing}, 67:\penalty0 106--135, 2005.

\bibitem[Edelman et~al.(1998)Edelman, Arias, and Smith]{edelman1998geometry}
Alan Edelman, Tom{\'a}s~A Arias, and Steven~T Smith.
\newblock The geometry of algorithms with orthogonality constraints.
\newblock \emph{SIAM journal on Matrix Analysis and Applications}, 20\penalty0
  (2):\penalty0 303--353, 1998.

\bibitem[Boumal and Absil(2011)]{boumal2011rtrmc}
Nicolas Boumal and Pierre-antoine Absil.
\newblock Rtrmc: A riemannian trust-region method for low-rank matrix
  completion.
\newblock In \emph{Advances in neural information processing systems}, pages
  406--414, 2011.

\bibitem[Theis et~al.(2009)Theis, Cason, and Absil]{theis2009soft}
Fabian~J Theis, Thomas~P Cason, and P-A Absil.
\newblock Soft dimension reduction for ica by joint diagonalization on the
  stiefel manifold.
\newblock In \emph{International Conference on Independent Component Analysis
  and Signal Separation}, pages 354--361. Springer, 2009.

\bibitem[Wen and Yin(2013)]{wen2013feasible}
Zaiwen Wen and Wotao Yin.
\newblock A feasible method for optimization with orthogonality constraints.
\newblock \emph{Mathematical Programming}, 142\penalty0 (1-2):\penalty0
  397--434, 2013.

\bibitem[Niu et~al.(2014)Niu, Dy, and Jordan]{niu2014iterative}
Donglin Niu, Jennifer~G Dy, and Michael~I Jordan.
\newblock Iterative discovery of multiple alternativeclustering views.
\newblock \emph{IEEE transactions on pattern analysis and machine
  intelligence}, 36\penalty0 (7):\penalty0 1340--1353, 2014.

\bibitem[Von~Luxburg(2007)]{von2007tutorial}
Ulrike Von~Luxburg.
\newblock A tutorial on spectral clustering.
\newblock \emph{Statistics and computing}, 17\penalty0 (4):\penalty0 395--416,
  2007.

\bibitem[Wright and Nocedal(1999)]{wright1999numerical}
Stephen Wright and Jorge Nocedal.
\newblock Numerical optimization.
\newblock \emph{Springer Science}, 35:\penalty0 67--68, 1999.

\bibitem[Dheeru and Karra~Taniskidou(2017)]{Dua:2017}
Dua Dheeru and Efi Karra~Taniskidou.
\newblock {UCI} machine learning repository, 2017.
\newblock URL \url{http://archive.ics.uci.edu/ml}.

\bibitem[Wolberg(1992)]{breastcancer}
William~H Wolberg.
\newblock Wisconsin breast cancer dataset.
\newblock \emph{University of Wisconsin Hospitals}, 1992.

\bibitem[Bay et~al.(2000)Bay, Kibler, Pazzani, and Smyth]{bay2000uci}
Stephen~D Bay, Dennis Kibler, Michael~J Pazzani, and Padhraic Smyth.
\newblock The uci kdd archive of large data sets for data mining research and
  experimentation.
\newblock \emph{ACM SIGKDD Explorations Newsletter}, 2\penalty0 (2):\penalty0
  81--85, 2000.

\bibitem[Deng(2012)]{deng2012mnist}
Li~Deng.
\newblock The mnist database of handwritten digit images for machine learning
  research [best of the web].
\newblock \emph{IEEE Signal Processing Magazine}, 29\penalty0 (6):\penalty0
  141--142, 2012.

\bibitem[par()]{particleoft}
Particles of tessellations.
\newblock \url{http://en.tessellations-nicolas.com/}.
\newblock Accessed: 2017-04-25.

\bibitem[Boumal et~al.(2014)Boumal, Mishra, Absil, and Sepulchre]{manopt}
N.~Boumal, B.~Mishra, P.-A. Absil, and R.~Sepulchre.
\newblock {M}anopt, a {M}atlab toolbox for optimization on manifolds.
\newblock \emph{Journal of Machine Learning Research}, 15:\penalty0 1455--1459,
  2014.
\newblock URL \url{http://www.manopt.org}.

\bibitem[Cortes et~al.(2012)Cortes, Mohri, and
  Rostamizadeh]{cortes2012algorithms}
Corinna Cortes, Mehryar Mohri, and Afshin Rostamizadeh.
\newblock Algorithms for learning kernels based on centered alignment.
\newblock \emph{Journal of Machine Learning Research}, 13\penalty0
  (Mar):\penalty0 795--828, 2012.

\bibitem[Strehl and Ghosh(2002)]{strehl2002cluster}
Alexander Strehl and Joydeep Ghosh.
\newblock Cluster ensembles---a knowledge reuse framework for combining
  multiple partitions.
\newblock \emph{Journal of machine learning research}, 3\penalty0
  (Dec):\penalty0 583--617, 2002.

\bibitem[Jones et~al.(2001--)Jones, Oliphant, Peterson, et~al.]{numpy}
Eric Jones, Travis Oliphant, Pearu Peterson, et~al.
\newblock {SciPy}: Open source scientific tools for {Python}, 2001--.
\newblock URL \url{http://www.scipy.org/}.
\newblock [Online; accessed <today>].

\bibitem[Buitinck et~al.(2013)Buitinck, Louppe, Blondel, Pedregosa, Mueller,
  Grisel, Niculae, Prettenhofer, Gramfort, Grobler, Layton, VanderPlas, Joly,
  Holt, and Varoquaux]{sklearn_api}
Lars Buitinck, Gilles Louppe, Mathieu Blondel, Fabian Pedregosa, Andreas
  Mueller, Olivier Grisel, Vlad Niculae, Peter Prettenhofer, Alexandre
  Gramfort, Jaques Grobler, Robert Layton, Jake VanderPlas, Arnaud Joly, Brian
  Holt, and Ga{\"{e}}l Varoquaux.
\newblock {API} design for machine learning software: experiences from the
  scikit-learn project.
\newblock In \emph{ECML PKDD Workshop: Languages for Data Mining and Machine
  Learning}, pages 108--122, 2013.

\bibitem[Gretton et~al.(2012)Gretton, Borgwardt, Rasch, Sch{\"o}lkopf, and
  Smola]{gretton2012kernel}
Arthur Gretton, Karsten~M Borgwardt, Malte~J Rasch, Bernhard Sch{\"o}lkopf, and
  Alexander Smola.
\newblock A kernel two-sample test.
\newblock \emph{Journal of Machine Learning Research}, 13\penalty0
  (Mar):\penalty0 723--773, 2012.

\bibitem[Knyazev and Zhu(2012)]{knyazev2012principal}
Andrew~V Knyazev and Peizhen Zhu.
\newblock Principal angles between subspaces and their tangents.
\newblock 2012.

\end{thebibliography}
\bibliographystyle{unsrtnat}

\clearpage
\begin{appendices}
\section{Kernel Definitions }
Here we provide the definition of each kernel with relation to the projection matrix $W$ in terms of the kernel and as a function of $\beta=\mathbf{a}W W^T \mathbf{b}$.\\
\textbf{Linear Kernel}
\begin{equation}
    k(x_i,x_j) = x_i^TWW^Tx_j, \hspace{1cm} f(\beta) = \beta.
    \label{eq:linear_kernel}
\end{equation}

\textbf{Polynomial Kernel}
\begin{equation}
    k(x_i,x_j) = (x_i^TWW^Tx_j + c)^p, \hspace{1cm} f(\beta) = (\beta + c)^p.
    \label{eq:poly_kernel}
\end{equation}

\textbf{Gaussian Kernel}
\begin{equation}
    k(x_i,x_j) = e^{-\frac{(x_i-x_j)^TW W^T(x_i-x_j)}{2\sigma^2}},
    \hspace{1cm}
    f(\beta) = e^{-\frac{\beta}{2\sigma^2}}.
    \label{eq:gaussian_kernel}
\end{equation}

\textbf{Squared Kernel}
\begin{equation}
    k(x_i,x_j) = (x_i-x_j)^TW W^T(x_i-x_j), \hspace{1cm} f(\beta) = \beta.
    \label{eq:squared_kernel}
\end{equation}

\textbf{Multiquadratic Kernel}
\begin{equation}
    k(x_i,x_j) = \sqrt{(x_i-x_j)^TW W^T(x_i-x_j) + c^2}, \hspace{1cm} f(\beta) = \sqrt{\beta + c^2}.
    \label{eq:multiquadratic_kernel}
\end{equation}
\label{app:to_ism_family}
\end{appendices}

\begin{appendices}
\section{Derivation for each $\Phi_0$}
\label{app:deriv_phi_0}
Using Eq.~(\ref{eq:initial_eq_flipped}), we know that
    \begin{equation}
        \Phi_0 = \text{sign}(\mu) \sum_{i,j} \Gamma_{i,j} A_{i,j}.
    \end{equation}
If $\mathbf{a}$ and $\mathbf{b}$ are both defined as $x_i-x_j$, then 
    \begin{equation}
        \Phi_0 = \text{sign}(4 \mu) X^T(D_\Gamma - \Gamma) X.
        \label{eq:phi_0_form_1}
    \end{equation}
However, if $\mathbf{a}$ and $\mathbf{b}$ are defined as $(x_i,x_j)$, then 
    \begin{equation}
        \Phi_0 = \text{sign}(2 \mu) X^T\Gamma X.
        \label{eq:phi_0_form_2}
    \end{equation}
Therefore, to compute $\Phi_0$, the key is to first determine the ($\mathbf{a}$ , $\mathbf{b}$) based on the kernel and then find $\mu$ to determine the sign.

\textbf{$\mathbf{\Phi_0}$ for the Linear Kernel: }
With a Linear Kernel, $(\mathbf{a},\mathbf{b})$ uses $(x_i,x_j)$, therefore Eq.~(\ref{eq:phi_0_form_2}) is use. Since $f(\beta) = \beta$, the sign of the gradient with respect to $\beta$ is 
    \begin{equation}
        \text{sign}(2 \nabla_\beta f(\beta)) = \text{sign}(2) = 1.
    \end{equation}
Therefore, 
    \begin{equation}
        \Phi_0 = X^T\Gamma X.
    \end{equation}

\textbf{$\mathbf{\Phi_0}$ for the Polynomial Kernel: }
With a Polynomial Kernel, $(\mathbf{a},\mathbf{b})$ uses $(x_i,x_j)$, therefore Eq.~(\ref{eq:phi_0_form_2}) is use. Since $f(\beta) = (\beta + c)^p$, the sign of the gradient with respect to $\beta$ is 
    \begin{equation}
        \text{sign}(2 \nabla_\beta f(\beta)) = \text{sign}(2p(\beta+c)^{p-1}) = 1.
    \end{equation}
Therefore, 
    \begin{equation}
        \Phi_0 = X^T\Gamma X.
    \end{equation}

\textbf{$\mathbf{\Phi_0}$ for the Gaussian Kernel: }
With a Gaussian Kernel, $(\mathbf{a},\mathbf{b})$ uses $x_i-x_j$, therefore Eq.~(\ref{eq:phi_0_form_1}) is use. Since $f(\beta) = e^{-\frac{\beta}{2\sigma^2}}$, the sign of the gradient with respect to $\beta$ is 
    \begin{equation}
        \text{sign}(4 \nabla_\beta f(\beta)) = \text{sign}(-\frac{4}{2\sigma^2}e^{-\frac{\beta}{2\sigma^2}}) = -1.
    \end{equation}
Therefore, 
    \begin{equation}
        \Phi_0 =  -X^T(D_\Gamma - \Gamma)X.
    \end{equation}

\textbf{$\mathbf{\Phi_0}$ for the RBF Relative Kernel: }
With a RBF Relative Kernel, it is easier to start with the Lagrangian once we have approximated relative Kernel with the 2nd order Taylor expansion as
    \begin{equation}
        \mathcal{L} \approx
            -\sum_{i,j} \Gamma_{i,j} 
            \left[ 1 + \Tr(W^T(-\frac{1}{\sigma_i \sigma_j}A_{i,j})W) \right]
            - \Tr \left[ \Lambda(W^TW - I) \right].
    \end{equation}
The gradient of the Lagrangian is therefore
    \begin{equation}
        \nabla_W \mathcal{L} \approx
            \left[ \sum_{i,j} \Gamma_{i,j} (\frac{2}{\sigma_i \sigma_j}A_{i,j}) \right] W
            - 2 W \Lambda.
    \end{equation}    
Setting the gradient to 0, we get 
     \begin{equation}
            \left[ \sum_{i,j} (\frac{1}{\sigma_i \sigma_j} \Gamma_{i,j} A_{i,j}) \right] W
            = W \Lambda.
    \end{equation}       
If we let $\Sigma_{i,j} = \frac{1}{\sigma_i \sigma_j}$ and $\Psi = \Sigma \odot \Gamma$, then we end up with     
     \begin{equation}
        4 \left[ X^T (D_{\Psi} - \Psi) X \right] W
            = W \Lambda.
    \end{equation}       
This equation requires $W$ to be the eigenvectors associated with the smallest eigenvalues. We flip the sign so the most dominant eigenvectors are the solution. Therefore, we define $\Phi$ as 
     \begin{equation}
        \Phi = -X^T (D_{\Psi} - \Psi) X
    \end{equation}       
\textbf{$\mathbf{\Phi_0}$ for the Squared Kernel: }
With a Squared Kernel, $(\mathbf{a},\mathbf{b})$ uses $x_i-x_j$, therefore Eq.~(\ref{eq:phi_0_form_1}) is use. Since $f(\beta) = \beta$, the sign of the gradient with respect to $\beta$ is 
    \begin{equation}
        \text{sign}(4 \nabla_\beta f(\beta)) = \text{sign}(4) = 1.
    \end{equation}
Therefore, 
    \begin{equation}
        \Phi_0 = X^T(D_\Gamma - \Gamma)X.
    \end{equation}

\textbf{$\mathbf{\Phi_0}$ for the Multiquadratic Kernel: }
With a Multiquadratic Kernel, $(\mathbf{a},\mathbf{b})$ uses $x_i-x_j$, therefore Eq.~(\ref{eq:phi_0_form_1}) is use. Since $f(\beta) = \sqrt{\beta + c^2}$, the sign of the gradient with respect to $\beta$ is 
    \begin{equation}
        \text{sign}(4 \nabla_\beta f(\beta)) = \text{sign}(\frac{4}{2}(\beta + c^2)^{-1/2}) = 1.
    \end{equation}
Therefore, 
    \begin{equation}
        \Phi_0 =  X^T(D_\Gamma - \Gamma)X.
    \end{equation}

\end{appendices}

\begin{appendices}
\section{Derivation for each $\Phi$}
\label{app:deriv_phi}
Using Eq.~(\ref{eq:def_of_pseudo_kernel}), we know that
    \begin{equation}
        \Phi = \frac{1}{2}\sum_{i,j} \Gamma_{i,j}[\nabla_\beta f(\beta)] A_{i,j}. 
    \end{equation}
If we let $\Psi=\Gamma_{i,j} [\nabla_\beta f(\beta)]$ then $\Phi$ can also be written as
    \begin{equation}
        \Phi = \frac{1}{2}\sum_{i,j} \Psi_{i,j} A_{i,j}. 
    \end{equation}
If $\mathbf{a}$ and $\mathbf{b}$ are both defined as $x_i-x_j$, then 
    \begin{equation}
        \Phi = 2 X^T(D_\Psi - \Psi) X.
        \label{eq:phi_form_1}
    \end{equation}
However, if $\mathbf{a}$ and $\mathbf{b}$ are defined as $(x_i,x_j)$, then 
    \begin{equation}
        \Phi = X^T\Psi X.
        \label{eq:phi_form_2}
    \end{equation}
Therefore, to compute $\Phi$, the key is to first determine the ($\mathbf{a}$ , $\mathbf{b}$) based on the kernel and then find the appropriate $\Psi$.

\textbf{$\mathbf{\Phi}$ for the Linear Kernel: }
With a Linear Kernel, $(\mathbf{a},\mathbf{b})$ uses $(x_i,x_j)$, therefore Eq.~(\ref{eq:phi_form_2}) is use. Since $f(\beta) = \beta$, $\Phi$ becomes
    \begin{equation}
        \Phi = \frac{1}{2}\sum_{i,j} \Gamma_{i,j}[\nabla_\beta f(\beta)] A_{i,j} 
        = \frac{1}{2}\sum_{i,j} \Gamma_{i,j} A_{i,j}.
    \end{equation}
Since, we are only interested in the eigenvectors of $\Phi$ only the sign of the constants are necessary. Therefore, 
    \begin{equation}
        \Phi = \text{sign}(1) X^T\Gamma X = X^T\Gamma X.
    \end{equation}

\textbf{$\mathbf{\Phi}$ for the Polynomial Kernel: }
With a Polynomial Kernel, $(\mathbf{a},\mathbf{b})$ uses $(x_i,x_j)$, therefore Eq.~(\ref{eq:phi_form_2}) is use. Since $f(\beta) = (\beta + c)^p$, $\Phi$ becomes 
    \begin{equation}
         \Phi = \frac{1}{2}\sum_{i,j} \Gamma_{i,j}[\nabla_\beta f(\beta)] A_{i,j} 
         = \frac{1}{2}\sum_{i,j} \Gamma_{i,j}[p(\beta+c)^{p-1}] A_{i,j}.
    \end{equation}
Since $p$ is a constant, and $K_{XW,p-1} = (\beta+c)^{p-1}$ is the polynomial kernel itself with power of $(p-1)$, $\Psi$ becomes
    \begin{equation}
        \Psi = \Gamma \odot K_{XW,p-1},
    \end{equation}
and
    \begin{equation}
        \Phi = \text{sign}(p) X^T \Psi X = X^T \Psi X
    \end{equation}
\textbf{$\mathbf{\Phi}$ for the Gaussian Kernel: }
With a Gaussian Kernel, $(\mathbf{a},\mathbf{b})$ uses $x_i-x_j$, therefore Eq.~(\ref{eq:phi_0_form_1}) is use. Since $f(\beta) = e^{-\frac{\beta}{2\sigma^2}}$, $\Phi$ becomes 
    \begin{equation}
        \Phi = \frac{1}{2} \sum_{i,j} \Gamma_{i,j}[\nabla_\beta f(\beta)] A_{i,j} 
         = \frac{1}{2} \sum_{i,j} \Gamma_{i,j}[-\frac{1}{2\sigma^2}e^{-\frac{\beta}{2\sigma^2}}] A_{i,j}= -\frac{1}{4\sigma^2} \sum_{i,j} \Gamma_{i,j}[K_{XW}]_{i,j} A_{i,j}. 
    \end{equation}
If we let $\Psi=\Gamma \odot K_{XW}$, then
    \begin{equation}
        \Phi =  \text{sign}(-\frac{2}{4\sigma^2}) X^T(D_\Psi - \Psi)X = -X^T(D_\Psi - \Psi)X.
    \end{equation}
    
\textbf{$\mathbf{\Phi}$ for the Squared Kernel: }
With a Squared Kernel, $(\mathbf{a},\mathbf{b})$ uses $x_i-x_j$, therefore Eq.~(\ref{eq:phi_0_form_1}) is use. Since $f(\beta) = \beta$, $\Phi$ becomes 
    \begin{equation}
         \Phi = \frac{1}{2} \sum_{i,j} \Gamma_{i,j}[\nabla_\beta f(\beta)] A_{i,j} 
         = \frac{1}{2} \sum_{i,j} \Gamma_{i,j}A_{i,j}.    
    \end{equation}
Therefore, 
    \begin{equation}
        \Phi = \text{sign}(1) X^T(D_\Gamma - \Gamma)X =  X^T(D_\Gamma - \Gamma)X.
    \end{equation}   
    
\textbf{$\mathbf{\Phi}$ for the Multiquadratic Kernel: }
With a Multiquadratic Kernel, $(\mathbf{a},\mathbf{b})$ uses $x_i-x_j$, therefore Eq.~(\ref{eq:phi_0_form_1}) is use. Since $f(\beta) = \sqrt{\beta+c^2}$, $\Phi$ becomes
    \begin{equation}
        \Phi = \frac{1}{2} \sum_{i,j} \Gamma_{i,j}[\nabla_\beta f(\beta)] A_{i,j} 
         = \frac{1}{2} \sum_{i,j} \Gamma_{i,j}[\frac{1}{2}(\beta+c^2)^{-1/2}] A_{i,j}= \frac{1}{4} \sum_{i,j} \Gamma_{i,j}[K_{XW}]_{i,j}^{(-1)} A_{i,j}. 
    \end{equation}
If we let $\Psi=\Gamma \odot K_{XW}^{(-1)}$, then
    \begin{equation}
        \Phi =  \text{sign}(\frac{1}{4}) X^T(D_\Psi - \Psi)X = X^T(D_\Psi - \Psi)X.
    \end{equation}   
    
\textbf{$\mathbf{\Phi_0}$ for the RBF Relative Kernel: }
With a RBF Relative Kernel, we start with the initial Lagrangian 
    \begin{equation}
        \mathcal{L} = \sum_{i,j} \Gamma_{i,j} e^{-\frac{Tr(W^TA_{i,j}W)}{2\sigma_i \sigma_j}} -
        \Tr(\Lambda(W^TW - I))
    \end{equation}
where the gradient becomes
    \begin{equation}
        \nabla_W \mathcal{L} = -\sum_{i,j} 
        \frac{1}{\sigma_i \sigma_j} \Gamma_{i,j} e^{-\frac{Tr(W^TA_{i,j}W)}{2\sigma_i \sigma_j}}
        A_{i,j} W - 2 W \Lambda.
    \end{equation}
If we let $\Sigma_{i,j} = \frac{1}{\sigma_i \sigma_j}$ then we get 
     \begin{equation}
        \nabla_W \mathcal{L} = - \sum_{i,j} 
        \Psi_{i,j} 
        A_{i,j} W - 2 W \Lambda,
    \end{equation}   
where $\Psi_{i,j}=\Sigma_{i,j} \Gamma_{i,j} K_{XW_{i,j}}$. If we apply Appendix \ref{app:xi-xj} and set the gradient to 0, then we get
     \begin{equation}
        -4 \left[ X^T (D_{\Psi} - \Psi) X \right] 
        W = 2 W \Lambda.
    \end{equation}   
From here, we see that it has the same form as the Gaussian kernel, with $\Psi$ defined as $\Psi = \Sigma \odot \Gamma \odot K_{XW}$.

This equation requires $W$ to be the eigenvectors associated with the smallest eigenvalues. We flip the sign so the most dominant eigenvectors are the solution. Therefore, we define $\Phi$ as 
     \begin{equation}
        \Phi = X^T (D_{\Psi} - \Psi) X
    \end{equation}       

\end{appendices}

\begin{appendices}
\section{Proof for Theorem \ref{thm:stationary}}
\label{app:theorem_1_proof}

The main body of the proof is organized into two lemmas where the 1st lemma will prove the 1st order condition and the 2nd lemma will prove the 2nd order condition. For convenience, we included the 2nd Order Necessary Condition~\cite{wright1999numerical} in Appendix~\ref{app:Nocedal_Wright}. We also convert the optimization problem into a standard minimization form where we solve

\begin{equation}
    \underset{W}{\min} -\Tr ( \Gamma K_{X W}) \hspace{0.4cm} \text{s.t.} \hspace{0.2cm} W^T W = I.
\end{equation}

The proof is initialized by manipulating the different kernels into a common form. If we let $\beta = a(x_i,x_j)W W^{T} b(x_i, x_j)$, then the kernels in this family can be expressed as $f(\beta)$. This common form allows a universal proof that works for all kernels that belongs to the ISM family. Depending on the kernel, the definition of $f$, $a(x_i,x_j)$ and $b(x_i,x_j)$ are listed in Table~\ref{table:kernels_conversion}. Kernels in this form are functions of the Grassmannian $WW^T$. 

\begin{table}[h]
    \centering
    \begin{tabular}{l|c|c|c}
    Name & $f(\beta)$ & $a(x_i,x_j)$ & $b(x_i, x_j)$\\ \hline
    Linear & $\beta$ & $x_i$ & $x_j$ \\
    Polynomial & $(\beta + c)^p$ & $x_i$ & $x_j$ \\
    Gaussian &  $e^\frac{-\beta}{2\sigma^2}$ & $x_i-x_j$ & $x_i-x_j$ \\
    Squared & $\beta$ & $x_i-x_j$ & $x_i-x_j$ \\
    \end{tabular}
    \caption{Common components of different Kernels.}
    \label{table:kernels_conversion}
\end{table}
 
\begin{lemma} \label{basic_lemma}
    Given $\mathcal{L}$ as the Lagrangian of Eq. (\ref{eq:obj_1}), if $W^{\ast}$ is a fixed point of Algorithm \ref{alg:ism}, and $\Lambda^\ast$ is a diagonal matrix of its corresponding eigenvalues, then
    \begin{align}
        &\nabla_W \mathcal{L} (W^{\ast}, \Lambda^{\ast}) = 0, \label{eq:1st_W}\\
        &\nabla_{\Lambda} \mathcal{L} (W^{\ast}, \Lambda^{\ast}) = 0. \label{eq:1st_lambda}
    \end{align}
\end{lemma}

\begin{proof}
Since $\Tr(\Gamma K_{XW})=\sum_{i,j} \Gamma_{i,j}K_{XW_{i,j}}$, where the subscript indicates the  $i,j$th element of the associated matrix.
If we let $\mathbf{a} = a(x_i, x_j), \mathbf{b} = b(x_i, x_j)$, the Lagrangian of Eq. (\ref{eq:obj_1}) becomes
    \begin{align}
    \begin{split}
        \mathcal{L}(W, \Lambda) = -\sum_{ij}
        \Gamma_{ij} f(\mathbf{a}^{T}WW^{T}\mathbf{b}) 
        - \Tr[\Lambda(W^{T}W-I)]. \label{eq:Lagrangian}
    \end{split}
    \end{align}
The gradient of the Lagrangian with respect to $W$ is
    \begin{align}
    \begin{split}
    \nabla_W \mathcal{L}(W, \Lambda) = 
    -\sum_{ij}{\Gamma_{ij} f'(\mathbf{a}^{T}WW^{T}\mathbf{b})} (\mathbf{b}\mathbf{a}^T + \mathbf{a}\mathbf{b}^T)W - 2 W \Lambda.
    \end{split}
    \label{eq:lagrangian_gradient}
    \end{align}
If we let $A_{i,j}=\mathbf{ba}^T+\mathbf{ab}^T$ then setting $\nabla_W \mathcal{L}(W,\Lambda)$ of Eq. (\ref{eq:lagrangian_gradient}) to 0 yields the relationship
    \begin{align}
    \begin{split}
    0 = \left[-\frac{1}{2}\sum_{ij}{\Gamma_{ij} f'(\mathbf{a}^{T}WW^{T}\mathbf{b})} A_{i,j}\right] W - W \Lambda.
    \end{split}\label{eq:eig_decomp_1}
    \end{align}
Since $f'(\mathbf{a}^{T}WW^{T}\mathbf{b})$ is a scalar value that depends on indices $i,j$, we multiply it by $-\frac{1}{2}\Gamma_{i,j}$ to form a new variable $\Psi_{i,j}$. Then Eq. (\ref{eq:eig_decomp_1}) can be rewritten as
    \begin{align}
    \begin{split}
    \left[\sum_{ij}{\Psi_{ij}} A_{i,j}\right] W = W \Lambda.
    \end{split}\label{eq:eig_decomp_2}
    \end{align}
To match the form shown in Table~\ref{table:phis}, Appendix~\ref{app:xixj} further showed that if $\mathbf{a}$ and $\mathbf{b}$ is equal to $x_i$ and $x_j$, then 
    \begin{align}
    \begin{split}
    \left[\sum_{ij}{\Psi_{ij}} A_{i,j}\right] = 2X^T\Psi X.
    \end{split}\label{eq:xpsix}
    \end{align}
From Appendix~\ref{app:xi-xj}, if $\mathbf{a}$ and $\mathbf{b}$ are equal to $x_i-x_j$, then 
    \begin{align}
    \begin{split}
    \left[\sum_{ij}{\Psi_{ij}} A_{i,j}\right] = 4 X^T[D_\Psi - \Psi]X.
    \end{split}\label{eq:xDpsix}
    \end{align}   
    
If we let $\Phi = \left[\sum_{ij}{\Psi_{ij}} A_{i,j}\right]$, it yields the relationship $\Phi W = W \Lambda$
where the eigenvectors of $\Phi$ satisfies the 1st order condition of $\nabla_W \mathcal{L}(W^\ast,\Lambda^\ast)=0$. The gradient with respect to $\Lambda$ yields the expected constraint    \begin{align}
    \nabla_\Lambda \mathcal{L} = W^TW - I.
    \end{align}
Since the eigenvectors of $\Phi$ is orthonormal, the condition $\nabla_\Lambda \mathcal{L} = 0 = W^TW - I$ is also satisfied. Observing these 2 properties, Lemma~\ref{basic_lemma} confirms that the eigenvectors of $\Phi$ also satisfies the 1st order condition from Eq. (\ref{eq:obj_1}).

\end{proof}

\begin{lemma} \label{eq:2nd_lemma}
    Given a full rank $\Phi$, an eigengap defined by Eq.~(\ref{eq:final_conclusion}), and $W^*$ as the fixed point of Algorithm~\ref{alg:ism}, then
    \begin{align} 
    \begin{split}
     \tmop{Tr}( Z^T &\nabla_{W W}^2 \mathcal{L}(W^{\ast}, \Lambda^{\ast}) Z) \geq 0 \\ 
     &\tmop{for}   \tmop{all} Z \neq 0 , \tmop{with}  \nabla h (W^{\ast})^T Z = 0.   \label{eq:2nd_W} 
    \end{split}
    \end{align}   
\end{lemma}

\begin{proof}
To proof Lemma 2, we must relate the concept of eigengap to the conditions of
\begin{equation}
  \tmop{Tr} ( Z^T \nabla^2_{W W} \mathcal{L} ( W^{\ast}, \Lambda^{\ast}) Z)
  \geq 0 \nocomma \begin{array}{lllll}
    & \forall & Z \neq 0 & \tmop{with} & \nabla h ( W^{\ast})^T Z = 0
    \label{eq:orig_inequality}
  \end{array} . 
\end{equation}
Given the constraint $h ( W) = W^T W - I$, we start by computing the constrain
$\nabla h ( W^{\ast})^T Z = 0$. Given
\begin{equation}
  \nabla h ( W^{\ast})^T Z = \begin{array}{l}
    \lim\\
    t \rightarrow 0
  \end{array} \frac{\partial}{\partial t} h ( W \noplus + t Z),
\end{equation}
the constraint becomes
\begin{equation}
  \begin{array}{lll}
    \nabla h ( W^{\ast})^T Z = 0 & = & \begin{array}{l}
      \lim\\
      t \rightarrow 0
    \end{array} \frac{\partial}{\partial t} [ ( W \noplus + t Z)^T ( W \noplus
    + t Z) - I],\\
    0 & = & \begin{array}{l}
      \lim\\
      t \rightarrow 0
    \end{array} \frac{\partial}{\partial t} [ ( W^T \noplus W + t W^T Z + t
    Z^T W + t^2 Z^T Z) - I],\\
    0 & = & \begin{array}{l}
      \lim\\
      t \rightarrow 0
    \end{array} W^T Z + Z^T W + 2 t Z^T Z.
  \end{array}
\end{equation}
By setting the limit to 0, an important relationship emerges as
\begin{equation}
  \begin{array}{lll}
    0 & = & W^T Z + Z^T W.
    \label{eq:constraint_2nd}
  \end{array}
\end{equation}
Given a full rank operator $\Phi$, its eigenvectors must span the complete
$\mathcal{R}^d$ space. If we let $W$ and $\bar{W}$ represent the eigenvectors
chosen and not chosen respectively from Algorithm 1, and let $B$ and $\bar{B}$
be scambling matrices, then the matrix $Z \in \mathcal{R}^{d \times q}$ can be
rewritten as
\begin{equation}
  Z = W B + \bar{W}  \bar{B} .
  \label{eq:z_equal_to}
\end{equation}
It should be noted that since $W$ and $\bar{W}$ are eigenvalues of the
symmetric matrix $\Phi$, they are orthogonal to each other, i.e., $W^T \bar{W}
= 0$. Furthermore, if we replace $Z$ in Eq.~(\ref{eq:constraint_2nd}) with Eq.~(\ref{eq:z_equal_to}), we get the
condition
\begin{equation}
  \begin{array}{lll}
    0 & = & W^T ( W B + \bar{W}  \bar{B}) + ( W B + \bar{W}  \bar{B})^T W\\
    0 & = & B + B^T .
    \label{eq:antisym}
  \end{array}
\end{equation}
From Eq.~(\ref{eq:antisym}), we observe that $B$ must be a antisymmetric matrix because $B = -B^T$. Next, we work
to compute the inequality of of Eq.~(\ref{eq:orig_inequality}) by noting that
\begin{equation}
  \nabla^2_{W W} \mathcal{L} ( W, \Lambda) Z = \begin{array}{l}
    \lim\\
    t \rightarrow 0
  \end{array} \frac{\partial}{\partial t} \nabla \mathcal{L} ( W + t Z)
  \nosymbol .
\end{equation}
Also note that Lemma 1 has already computed $\nabla_W \mathcal{L} ( W)$ as
\begin{equation}
  \nabla_W \mathcal{L} ( W) = -\frac{1}{2}\left[ \sum_{i, j} \Gamma_{i, j} f' ( \beta)
  A_{i, j} \right] W - W \Lambda .
\end{equation}
Since we need $\nabla_W \mathcal{L}$ to be a function of $W + t Z$ with $t$ as
the variable, we change $\beta ( W)$ into $\beta ( W + t Z)$ with
\begin{equation}
  \begin{array}{lll}
    \beta ( W + t Z) & = & \tmmathbf{a} ( W + t Z) ( W + t Z)^T
    \tmmathbf{b},\\
    & = & \tmmathbf{a}^T W W^T \tmmathbf{b}+ [ \tmmathbf{a}^T ( W Z^T + Z
    W^T) \tmmathbf{b}] t + [ \tmmathbf{a}^T Z Z^T \tmmathbf{b}] t^2,\\
    & = & \beta + c_1 t + c_2 t^2,
    \label{eq:beta}
  \end{array}
\end{equation}
where $\beta$, $c_1$, and $c_2$ are constants with respect to $t$. Using the
$\beta$ from Eq.~(\ref{eq:beta}) with $\nabla_W \mathcal{L}$, we get
\begin{equation}
  \nabla^2_{W W} \mathcal{L} ( W, \Lambda) Z = \begin{array}{l}
    \lim\\
    t \rightarrow 0
  \end{array} \frac{\partial}{\partial t} \left[-\frac{1}{2} \sum_{i, j}  \Gamma_{i, j}
  f' ( \beta + c_1 t + c_2 t^2) A_{i, j} \right] ( W + t Z) - ( W + t Z)
  \Lambda .
\end{equation}
If we take the derivative with respect to $t$ and then set the limit to 0, we
get
\begin{equation}
  \nabla^2_{W W} \mathcal{L} ( W, \Lambda) Z = \left[-\frac{1}{2} \sum_{i, j} \Gamma_{i,
  j} f'' ( \beta) c_1 A_{i, j} \right] W + \left[ -\frac{1}{2} \sum_{i, j} \Gamma_{i, j}
  f' ( \beta) A_{i, j} \right] Z - Z \Lambda .
\end{equation}
Next, we notice the definition of $\Phi = -\frac{1}{2} \sum \Gamma_{i, j} f' ( \beta)
A_{i, j}$ \ from Lemma 1, the term $\tmop{Tr} ( Z^T \nabla^2_{W W} \mathcal{L}
( W, \Lambda) Z)$ can now be expressed as 3 separate terms as
\begin{equation}
\tmop{Tr} ( Z^T \nabla^2_{W W} \mathcal{L} ( W, \Lambda) Z) =\mathcal{T}_1
   +\mathcal{T}_2 +\mathcal{T}_3,
   \label{eq:3_terms}
\end{equation}
where
    \begin{align}
        \mathcal{T}_1 & = \tmop{Tr} \left( Z^T \left[ -\frac{1}{2}\sum_{i, j} \Gamma_{i,
     j} f'' ( \beta) c_1 A_{i, j} \right] W \right),\\
          \mathcal{T}_2 & = \tmop{Tr} ( Z^T \Phi Z),
          \label{eq:app:T2}
          \\
        \mathcal{T}_3 & = - \tmop{Tr} ( Z^T Z \Lambda).
        \label{eq:app:T3}
    \end{align}
Since $\mathcal{T}_1$ cannot be further simplified, the concentration will be
on $\mathcal{T}_2$ and $\mathcal{T}_3$. If we let $\bar{\Lambda}$ and
$\Lambda$ be the corresponding eigenvlaue matrices associated with $\bar{W}$
and $W$, by replacing $Z$ in $\mathcal{T}_2$ from Eq.~(\ref{eq:app:T2}), we get
\[ \begin{array}{lll}
     \tmop{Tr} ( Z^T \Phi Z) & = & \tmop{Tr} ( ( W B + \bar{W} \bar{B})^T \Phi
     ( W B + \bar{W} \bar{B}))\\
     & = & \tmop{Tr} ( \nobracket B^T W^T \Phi W B + \bar{B}^T \bar{W}^T \Phi
     W B + B^T W^T \Phi \bar{W} \bar{B} + \bar{B}^T \bar{W}^T \Phi \bar{W}
     \bar{B}) \nobracket\\
     & = & \tmop{Tr} ( \nobracket B^T W^T W \Lambda B + \bar{B}^T \bar{W}^T W
     \Lambda B + B^T W^T \bar{W} \bar{\Lambda} \bar{B} + \bar{B}^T \bar{W}^T
     \bar{W} \bar{\Lambda} \bar{B}) \nobracket\\
     & = & \tmop{Tr} ( \nobracket B^T \Lambda B + 0 + 0 + \bar{B}^T
     \bar{\Lambda} \bar{B}) \nobracket\\
     & = & \tmop{Tr} ( \nobracket B^T \Lambda B + \bar{B}^T \bar{\Lambda}
     \bar{B}) \nobracket .
   \end{array} \]
By replacing $Z$ from $\mathcal{T}_3$ from Eq.~(\ref{eq:app:T3}), we get
\[ \begin{array}{lll}
     - \tmop{Tr} ( Z^T Z \Lambda) & = & - \tmop{Tr} ( ( W B + \bar{W}
     \bar{B})^T ( W B + \bar{W} \bar{B}) \Lambda)\\
     & = & - \tmop{Tr} ( B^T W^T W B \Lambda + \bar{B}^T \bar{W}^T W B
     \Lambda + B^T W^T \bar{W} \bar{B} \Lambda + \bar{B}^T \bar{W}^T \bar{W}
     \bar{B} \Lambda)\\
     & = & - \tmop{Tr} ( B^T B \Lambda + 0 + 0 + \bar{B}^T \bar{B} \Lambda)\\
     & = & - \tmop{Tr} ( B \Lambda B^T + \bar{B} \Lambda \bar{B}^T) .
   \end{array} \]
The inequality that satisfies the 2nd order condition can now be written as 
\begin{equation}
  \tmop{Tr} ( B^T \Lambda B) + \tmop{Tr} (  \bar{B}^T \bar{\Lambda} \bar{B}) -
  \tmop{Tr} ( B \Lambda B^T) - \tmop{Tr} ( \bar{B} \Lambda \bar{B}^T)
  +\mathcal{T}_1 \geq 0.
  \label{eq:inequality_2}
\end{equation}
Since $B$ is an antisymmetric matrix, $B^T = - B$, and therefore $\tmop{Tr} (
B \Lambda B^T) = \tmop{Tr} ( B^T \Lambda B)$. From this Eq.~(\ref{eq:inequality_2}) can be
rewritten as
\begin{equation}
  \tmop{Tr} ( B^T \Lambda B) - \tmop{Tr} ( B^T \Lambda B) + \tmop{Tr} ( 
  \bar{B}^T \bar{\Lambda} \bar{B}) - \tmop{Tr} ( \bar{B} \Lambda \bar{B}^T)
  +\mathcal{T}_1 \geq 0.
\end{equation}
With the first two terms canceling each other out, the inequality can be
rewritten as
\begin{equation}
  \tmop{Tr} (  \bar{B}^T \bar{\Lambda} \bar{B}) - \tmop{Tr} ( \bar{B} \Lambda
  \bar{B}^T) \geq \mathcal{T}_1 .
\end{equation}
With this inequality, the terms can be further bounded by
\[ \tmop{Tr} ( \bar{B}^T \bar{\Lambda} \bar{B}) \geq \begin{array}{l}
     \min\\
     i
   \end{array} \bar{\Lambda}_i \tmop{Tr} ( \bar{B} \bar{B}^T) \]
\[ \tmop{Tr} ( \bar{B} \Lambda \bar{B}^T) \geq \begin{array}{l}
     \max\\
     j
   \end{array} \Lambda_j \tmop{Tr} ( \bar{B}^T \bar{B}) \]
Noting that since $\tmop{Tr} ( \bar{B} \bar{B}^T) = \tmop{Tr} ( \bar{B}^T
\bar{B})$, we treat it as a constant value of $\alpha$. With this, the inequality can be rewritten as
\[ \left( \begin{array}{l}
     \min\\
     i
   \end{array} \bar{\Lambda}_i - \begin{array}{l}
     \max\\
     j
   \end{array} \Lambda_j \right) \geq \frac{1}{\alpha} \mathcal{T}_1 . \]

Here, since $\frac{1}{\alpha}\mathcal{T}_1$ is simply a constant, we denote it as $\mathcal{C}$ to yield the final conclusion that

    \begin{equation}
    \left( \begin{array}{l}
     \min\\
     i
   \end{array} \bar{\Lambda}_i - \begin{array}{l}
     \max\\
     j
   \end{array} \Lambda_j \right) \geq \mathcal{C} . 
    \label{eq:final_conclusion}
    \end{equation}
    Eq.~(\ref{eq:final_conclusion}) concludes that to satisfy the 2nd order condition, the eigengap must be greater than $\mathcal{C}$. Therefore, given the choice of $q$ eigenvectors, the eigengap is maximized when the eigenvectors associated with the $q$ smallest eigenvalues are chosen as $W$. 
\end{proof}

    We note that it is customary for machine learning algorithms to look for the most dominant eigenvectors, crucially, many KDR algorithms follow this standard. Therefore, to maintain consistency, the $\Phi$ defined within the paper is actually the negative $\Phi$ from the proof. By flipping the sign, the eigenvectors associated with the smallest eigenvalues is now the most dominant eigenvectors. Hence, $\Phi$ within the paper is defined as
    
    \begin{equation}
        \Phi = \frac{1}{2}\sum_{ij}{\Gamma_{ij} f'(\mathbf{a}^{T}WW^{T}\mathbf{b})} A_{i,j}.
    \end{equation}

\end{appendices}
\begin{appendices}
\section{Computing the Hessian for the Taylor Series}
\label{app:Hessian_proof}
First we compute the gradient and the Hessian for $\beta(W)$ where
    \begin{align}
    \beta(W) &= a^TWW^Tb,\\
    \beta(W) &= \Tr(W^Tba^TW),\\
    \nabla_W \beta(W) &= [ba^T+ab^T]W,\\
    \nabla_{W,W} \beta(W) &= [ba^T+ab^T],\\
    \nabla_{W,W} \beta(W=0) &= [ba^T+ab^T].\\
    \end{align}
Next, we compute the gradient and Hessian for $f(\beta(W))$ where
    \begin{align}
    f(\beta(W)) &= f(a^TWW^Tb),\\
    f(\beta(W)) &= f(\Tr(W^Tba^TW)),\\
    f'(\beta(W)) &= \nabla_{\beta} f(\beta(W))[ba^T+ab^T]W = 
        \nabla_{\beta} f(\beta(W)) \nabla_W \beta(W)\\
    f''(\beta(W) &= \nabla_{\beta,\beta} f(\beta(W))[ba^T+ab^T]W(...) + \nabla_{\beta} f(\beta(W))[ba^T+ab^T]\\
    f''(\beta(W=0)) &= 0 + \nabla_{\beta} f(\beta(W))\nabla_{W,W} \beta(W=0)\\
    f''(\beta(W=0)) &= \nabla_{\beta} f(\beta(W))\nabla_{W,W} \beta(W=0)\\
    f''(0) &= \mu A_{i,j}.
    \end{align}
Using Taylor Series the gradient of the Lagrangian is approximately
    \begin{align}
    \nabla_{W} \mathcal{L} &\approx -\sum_{i,j} \Gamma_{i,j} f''(0) W - 2W\Lambda,\\
    \nabla_{W} \mathcal{L} &\approx -\mu \sum_{i,j} \Gamma_{i,j} A_{i,j} W - 2W\Lambda.
    \end{align}
Setting the gradient of the Lagrangian to 0 and combining the constant 2 to $\mu$, it yields the relationship
    \begin{align}
    \left[ -\mu \sum_{i,j} \Gamma_{i,j} A_{i,j} \right] W = W\Lambda.
    \end{align}
Here we note that $\mu$ is a constant. Therefore, only the sign will affect the eigenvector selection. With this, it yields    
    \begin{align}
    \left[ -\text{sign}(\mu) \sum_{i,j} \Gamma_{i,j} A_{i,j} \right] W = W\Lambda.
    \end{align}   
With this, the terms within the bracket become the initial $\Phi_0$ as
    \begin{align}
    \Phi_0 W = W\Lambda.
    \end{align}   
    
\end{appendices}

\begin{appendices}
\section{Derivation for Approximated $\Phi$}
\label{app:approx_phi}

We first convert the optimization problem into a standard minimization form where we solve

\begin{equation}
    \underset{W}{\min} -\Tr ( \Gamma K_{X W}) \hspace{0.4cm} \text{s.t.} \hspace{0.2cm} W^T W = I.
\end{equation}

Since the objective Lagrangian is non-convex, a solution can be achieved faster and more accurately if the algorithm is initialized at an intelligent starting point. Ideally, we wish to have a closed-form solution that yields the global optimal without any iterations. However, this is not possible since $\Phi$ is a function of $W$. ISM circumvents this problem by approximating the kernel using Taylor Series up to the 2nd order while expanding around 0. This approximation has the benefit of removing the dependency of $W$ for $\Phi$, therefore, a global minimum can be achieved using the approximated kernel. The ISM algorithm uses the global minimum found from the approximated kernel as the initialization point. Here, we provide a generalized derivation for the ISM kernel functions that are twice differentiable. First, we note that the 2nd order Taylor expansion for $f(\beta(W))$ around 0 is
    $f(\beta(W)) \approx f(0) + \frac{1}{2!}\Tr(W^T f''(0) W)$,
where the 1st order expansion around 0 is equal to 0. Therefore, the ISM Lagrangian can be approximated with
    \begin{align} 
    \begin{split} 
    \mathcal{L} = -\sum_{i,j} \Gamma_{i,j} \left[f(0)+\frac{1}{2!}\Tr(W^T f''(0) W)\right] 
    - \Tr(\Lambda(W^TW-I)),
    \end{split} 
    \end{align}   
where the gradient of the Lagrangian is
    \begin{align} 
    \begin{split} 
    \nabla_W \mathcal{L} = -\sum_{i,j} \Gamma_{i,j} f''(0) W - 2 W \Lambda.
    \end{split} \label{eq:approx_grad}
    \end{align}   
Next, we look at the kernel function $f(\beta(W))$ more closely. The Hessian is computed as
    \begin{align} 
    f'(\beta(W)) = \nabla_\beta f(\beta(W)) \nabla_W \beta(W),\\
    f''(\beta(W=0)) = \nabla_{\beta} f(\beta(0)) \nabla_{W,W} \beta(0).
    \end{align}   
Since we skipped several steps for the computation of the Hessian, refer to Appendix~\ref{app:Hessian_proof} for more detail. Because $\nabla_{\beta} f(\beta(0))$  is just a constant, we can bundle all constants into this term and refer to it as $\mu$. Since $\nabla_{W,W} \beta(0) = A_{i,j}$, the Hessian is simply $\mu A_{i,j}$ regardless of the kernel. By combining constants setting the gradient of Eq.~(\ref{eq:approx_grad}) to 0, we get the expression
    \begin{align} 
    \begin{split} 
    \left[-\sign(\mu) \sum_{i,j} \Gamma_{i,j} A_{i,j}\right] W = W \Lambda,
    \end{split} \label{eq:initial_eq}
    \end{align}   
where if we let $\Phi=-\sign(\mu) \sum_{i,j} \Gamma_{i,j} A_{i,j}$, we get a $\Phi$ that is not dependent on $W$. Therefore, a closed-form global minimum of the second-order approximation can be achieved. It should be noted that while the magnitude of $\mu$ can be ignored, the sign of $\mu$ cannot be neglected since it flips the sign of the eigenvalues of $\Psi$. Following Eq. (\ref{eq:initial_eq}), the initial $\Phi_0$ for each kernel is shown in Table~\ref{table:init_phis}. We also provide detailed proofs for each kernel in Appendix~\ref{app:deriv_phi_0}. 

It is important to note that based on proof of Theorem~\ref{thm:stationary} in Appendix  ~\ref{app:theorem_1_proof}, the $\Phi$ as defined from Eq.~(\ref{eq:initial_eq}) requires the optimal $W$ to be the eigenvectors of $\Phi$ that is associated with the smallest eigenvalues. This is equivalent to the most dominant eigenvectors of negative $\Phi$. To maintain consistency, the $\Phi$ defined with the paper is the negative $\Phi_0$ from this derivation, and therefore the $\Phi_0$ defined within the paper is 
    \begin{align} 
        \Phi=\sign(\mu) \sum_{i,j} \Gamma_{i,j} A_{i,j}.
        \label{eq:initial_eq_flipped}
    \end{align}   
\end{appendices}
\begin{appendices}
\section{Theorem 12.5 }\label{app:Nocedal_Wright}
\begin{lemma} 
  [Nocedal,Wright, Theorem 12.5~{\cite{wright1999numerical}}] (2nd Order Necessary Conditions)
    Consider the optimization problem:
  $ \min_{W : h (W) = 0} f (W), $
where $f : \mathbb{R}^{d \times q} \to \mathbb{R}$ and $h :
  \mathcal{R}^{d \times q} \to \mathbb{R}^{q \times q}$ are twice continuously
  differentiable. Let   $\mathcal{L}$ be the Lagrangian and $h(W)$ its equality constraint. Then, a local minimum must satisfy the following  conditions:
  \begin{subequations}
    \begin{align} 
    &\nabla_W \mathcal{L} (W^{\ast}, \Lambda^{\ast}) = 0, \label{eq:1st_W_nocedal}\\
    &\nabla_{\Lambda} \mathcal{L} (W^{\ast}, \Lambda^{\ast}) = 0, \label{eq:1st_lambda-nocedal}\\
    \begin{split}
     \tmop{Tr}( Z^T &\nabla_{W W}^2 \mathcal{L}(W^{\ast}, \Lambda^{\ast}) Z) \geq 0 \\ 
     &\tmop{for}   \tmop{all} Z \neq 0 , \tmop{with}  \nabla h (W^{\ast})^T Z = 0.   \label{eq:2nd_W_in_append} 
    \end{split}
    \end{align}   
  \end{subequations}
\end{lemma}
\end{appendices}

\begin{appendices}
\section{Derivation for $\sum_{i,j} \Psi_{i,j} A_{i,j}$ if $A_{i,j}=x_ix_j^T+x_jx_i^T$}
Since $\Psi$ is a symmetric matrix and $A_{i, j} = ( x_i x_j^T + x_j x_i^T)$, we first note that while $x_i x_j^T \ne x_j x_i^T$, it still hold that
\begin{equation}
    \sum_{i,j} \Psi_{i,j} x_i x_j^T = 
    \sum_{i,j} \Psi_{i,j} x_j x_i^T.
\end{equation}

Therefore,  we can rewrite the expression into
\[ \sum_{i, j} \Psi_{i, j} A_{i, j} = 2 \sum_{i, j}^n \Psi_{i, j} x_i x_j^T .
\]
If we expand the summation for $i = 1$, we get
\[ \begin{array}{lll}
     {}[ \Psi_{1, 1} x_1 x_1^T + \ldots + \Psi_{1, n} x_1 x_n^T] & = & x_1 [
     \Psi_{1, 1} x_1^T + \ldots + \Psi_{1, n} x_n^T]\\
     & = & x_1 \left[ \left[ \begin{array}{lll}
       x_1 & \ldots & x_n
     \end{array} \right] \left[ \begin{array}{l}
       \Psi_{1, 1}\\
       .\\
       \Psi_{1, n}
     \end{array} \right] \right]^T\\
     & = & x_1 \left[ \left[ \begin{array}{lll}
       \Psi_{1, 1} & \ldots & \Psi_{1, n}
     \end{array} \right] \left[ \begin{array}{l}
       x_1^T\\
       .\\
       x_n^T
     \end{array} \right] \right] .
   \end{array} \]
Now if we sum up all $i$, we get
\[ \begin{array}{lll}
     \Psi_{i, j} x_i x_j^T & = & x_1 \left[ \left[ \begin{array}{lll}
       \Psi_{1, 1} & \ldots & \Psi_{1, n}
     \end{array} \right] \left[ \begin{array}{l}
       x_1^T\\
       .\\
       x_n^T
     \end{array} \right] \right] + \ldots + x_n \left[ \left[
     \begin{array}{lll}
       \Psi_{n, 1} & \ldots & \Psi_{n, n}
     \end{array} \right] \left[ \begin{array}{l}
       x_1^T\\
       .\\
       x_n^T
     \end{array} \right] \right],\\
     & = & \left[ x_1 \left[ \begin{array}{lll}
       \Psi_{1, 1} & \ldots & \Psi_{1, n}
     \end{array} \right] + \ldots + x_n \left[ \begin{array}{lll}
       \Psi_{n, 1} & \ldots & \Psi_{n, n}
     \end{array} \right] \right] \left[ \begin{array}{l}
       x_1^T\\
       .\\
       x_n^T
     \end{array} \right],\\
     & = & \left[ \left[ \begin{array}{lll}
       x_1 & \ldots & x_n
     \end{array} \right] \left[ \begin{array}{l}
       \Psi_{1, 1}\\
       .\\
       \Psi_{n, 1}
     \end{array} \right] + \ldots + \left[ \begin{array}{lll}
       x_1 & \ldots & x_n
     \end{array} \right] \left[ \begin{array}{l}
       \Psi_{1, n}\\
       .\\
       \Psi_{n, n}
     \end{array} \right] \right] \left[ \begin{array}{l}
       x_1^T\\
       .\\
       x_n^T
     \end{array} \right],\\
     & = & \left[ \begin{array}{lll}
       x_1 & \ldots & x_n
     \end{array} \right] \left[ \begin{array}{lll}
       \left[ \begin{array}{l}
         \Psi_{1, 1}\\
         .\\
         \Psi_{n, 1}
       \end{array} \right] & \ldots & \left[ \begin{array}{l}
         \Psi_{1, n}\\
         .\\
         \Psi_{n, n}
       \end{array} \right]
     \end{array} \right] \left[ \begin{array}{l}
       x_1^T\\
       .\\
       x_n^T
     \end{array} \right] .
   \end{array} \]
Given that $X = \left[ \begin{array}{lll}
  x_1 & \ldots & x_n
\end{array} \right]^T$, the final expression becomes.
\[ 2 \sum_{i, j}^n \Psi_{i, j} x_i x_j^T = 2 X^T \Psi X. \]
\label{app:xixj}
\end{appendices}

\begin{appendices}
\section{Derivation for $\sum_{i,j} \Psi_{i,j} A_{i,j}$ if $A_{i,j}=(x_i-x_j)(x_i-x_j)^T+(x_i-x_j)(x_i-x_j)^T$}
Since $\Psi$ is a symmetric matrix, and $A_{i, j} = ( x_i - x_j) ( x_i -
x_j)^T + ( x_i - x_j) ( x_i - x_j)^T = 2 ( x_i - x_j) ( x_i - x_j)^T$, we can
rewrite the expression into
\[ \begin{array}{lll}
     \sum_{i, j} \Psi_{i, j} A_{i, j}  & = & 2 \sum_{i, j} \Psi_{i, j} ( x_i -
     x_j) ( x_i - x_j)^T\\
     & = & 2 \sum_{i, j} \Psi_{i, j} ( x_i x_i^T - x_j x_i^T - x_i x_j^T +
     x_j x_j^T)\\
     & = & 4 \sum_{i, j} \Psi_{i, j} ( x_i x_i^T - x_j x_i^T)\\
     & = & \left[ 4 \sum_{i, j} \Psi_{i, j} ( x_i x_i^T) \right] - \left[ 4
     \sum_{i, j} \Psi_{i, j} ( x_j x_i^T) \right] .
   \end{array} \]
If we expand the 1st term where $i = 1$, we get
\[ \sum_{i = 1, j}^n \Psi_{1, j} ( x_1 x_1^T) = \Psi_{1, 1} ( x_1 x_1^T) +
   \ldots + \Psi_{1, n} ( x_1 x_1^T) = \left[ \sum_{i = 1, j}^n \Psi_{1, j}
   \right] x_1 x_1^T . \]
From here, we notice that $\left[ \sum_{i = 1, j}^n \Psi_{1, j} \right]$ is
the degree $d_{i = 1}$ of $\Psi_{i = 1}$. Therefore, if we sum up all $i$
values we get
\[ \sum_{i, j} \Psi_{i, j} ( x_i x_i^T) = d_1 x_1 x_1^T + \ldots + d_n x_n
   x_n^T \nosymbol . \]
If we let $D_{\Psi}$ be the degree matrix of $\Psi$, then this expression
becomes
\[ 4 \sum_{i, j} \Psi_{i, j} ( x_i x_i^T) = 4 X^T D_{\Psi} X. \]
Since Appendix~\ref{app:xixj} has already proven the 2nd term, together we get
\[ 4 \sum_{i, j} \Psi_{i, j} ( x_i x_i^T) - 4 \sum_{i, j} \Psi_{i, j} ( x_j
   x_i^T) = 4 X^T D_{\Psi} X - 4 X^T \Psi X = 4 X^T [ D_{\Psi} - \Psi] X. \]
\label{app:xi-xj}
\end{appendices}

\begin{appendices}
\section{Dataset Details}
\label{app:data_detail}

\textbf{Wine. } 
    This dataset has 13 features and 178 samples. The features are continuous and heavily unbalanced in magnitude. During the experiments, the dimension is reduced down to 3 prior to performing supervised or unsupervised tasks. The dataset can be downloaded at \url{https://archive.ics.uci.edu/ml/datasets/wine.}

\textbf{Cancer. } 
    This dataset has 9 features and 683 samples. The features are discrete and unbalanced in magnitude. During the experiments, the dimension is reduced down to 2 prior to performing supervised or unsupervised tasks. The dataset can be downloaded at \url{https://archive.ics.uci.edu/ml/datasets/Breast+Cancer+Wisconsin+(Diagnostic)}.
    
 \textbf{Face. } 
    This dataset consists of images of 20 people in various poses. The 624 images are vectorized into 960 features. During the experiments, the dimension is reduced down to 20 prior to performing supervised or unsupervised tasks. This dataset is commonly used for alternative clustering since it can be clustered by the identity or the pose of the individuals. 
    The dataset can be downloaded at 
    \url{https://archive.ics.uci.edu/ml/datasets/CMU+Face+Images}.

 \textbf{MNIST. } 
    This dataset consists of 10,000 images of 10 characters in various orientations. The images are vectorized into 785 features. During the experiments, the dimension is reduced down to 10 prior to performing supervised or unsupervised tasks. The original MNIST dataset consists of 60,000 training samples and 10,000 test samples. We have decided to use the 10,000 test samples as our dataset. Since ISM have a memory complexity of $O(n^2)$, storing matrix size of 60,000 $\times$ 60,000 was beyond our computer's capability. We are actively conducting research into using the concept of coresets to alleviate the memory bottleneck. 
    The dataset can be downloaded at
    \url{http://yann.lecun.com/exdb/mnist/}
    
 \textbf{Flower. } 
    The Flower image is a dataset that allows for alternative ways to perform image segmentation. It is an image of 350x256 pixel. The RGB values of each pixel is taken as a single sample, with repeated samples removed. This results in a dataset of 256 samples and 3  features. The image is segmented into group of 2, represented by black and white. 
    The dataset can be downloaded at
    \url{http://en.tessellations-nicolas.com/}

\end{appendices}
\begin{appendices}
\section{Proof of Reformulating Eq.~(\ref{eq:obj_1}) into Quadratic Optimization}
\label{app:corollary}
Given 
    \begin{align}
        \min_W \Tr(W^T \Phi W) 
        \hspace{0.5cm} \text{s.t.} 
        \hspace{0.3cm}
        W^TW = I.
        \label{eq:obj_2}
    \end{align}
    
Here we proof that the local minimum for Eq.~(\ref{eq:obj_2}) is equivalent to a local minimum for Eq.~(\ref{eq:obj_1}). From Theorem~\ref{thm:stationary}, we establish that the $q$ minimizing eigenvectors of $\Phi \in \mathcal{R}^{d \times d}$ is a local minimum of Eq.~(\ref{eq:obj_1}). Therefore, the strategy of this proof is to show that the optimal solution for Eq.~(\ref{eq:obj_2}) is also the minimizing eigenvectors of $\Phi$.
\begin{proof}
Given Eq.~(\ref{eq:obj_2}), the Lagrangian of the objective is 
    \begin{align}
        \mathcal{L}(W) = \Tr(W^T \Phi W) - \Tr\left[ \Lambda(W^TW-I) \right].
        \label{eq:lagrangian_of_quadratic}
    \end{align}
Therefore, given a symmetric $\Phi$, the gradient of the Lagrangian becomes
    \begin{align}
        \nabla_W \mathcal{L}(W) = 2 \Phi W - 2 W \Lambda.
    \end{align}
Here, by setting the gradient to 0, we arrive to the definition of eigenvector where
    \begin{align}
        \Phi W =  W \Lambda,
    \end{align}
thereby proving that the eigenvector of $\Phi$ is also a stationary point for Eq.~(\ref{eq:obj_2}). The proof ends here if $W \in \mathcal{R}^{d \times d}$, however, if $W \in \mathcal{R}^{d \times q}$ where $q < d$, then we must also determine the appropriate $q$ eigenvectors to minimize the objective. Given $\bar{W} \in \mathcal{R}^{d \times d}$ as the full set eigenvectors, we replace $W$ from Eq.~(\ref{eq:lagrangian_of_quadratic}) with $\bar{W}$ to get
    \begin{align}
        \mathcal{L}(\bar{W}) = \Tr(\bar{W}^T \Phi \bar{W}) - \Tr\left[ \Lambda(\bar{W}^T\bar{W}-I) \right].
        \label{eq:bar_lagrangian_of_quadratic}
    \end{align}
Since $\bar{W}^T \bar{W} = I$ and $\Phi \bar{W} = W \Lambda$, we substitute these terms into Eq.~(\ref{eq:bar_lagrangian_of_quadratic}) to get
    \begin{align}
        \mathcal{L}(\bar{W}) = \Tr(\bar{W}^T \bar{W} \Lambda).
        \label{eq:bar_lagrangian_of_quadratic_2}
    \end{align}
If we let ${w_1, w_2,...,w_d}$ be the set of individual eigenvectors of $\Phi$ within $\bar{W}$ and ${\lambda_1, \lambda_2,...,\lambda_d}$ be their corresponding eigenvalues, then Eq.~(\ref{eq:bar_lagrangian_of_quadratic_2}) can be rewritten as
    \begin{align}
        \mathcal{L}(\bar{W}) = \lambda_1 w_1^T w_1 + \lambda_2 w_2^T w_2 + ... + 
        \lambda_n w_d^T w_d.
        \label{eq:bar_lagrangian_of_quadratic_3}
    \end{align}
Since the inner product of any eigenvector with itself ($w_i^T w_i$) is always equal to 1, the Lagrangian becomes the summation of its eigenvalues where 
    \begin{align}
        \mathcal{L}(\bar{W}) = \lambda_1 + \lambda_2 + ... + 
        \lambda_n .
        \label{eq:bar_lagrangian_of_quadratic_4}
    \end{align}
Therefore, the selection of a subset of eigenvectors is equivalent to keeping a subset of eigenvalues while setting the rest to 0 in  Eq.~(\ref{eq:bar_lagrangian_of_quadratic_4}). To minimize the Lagrangian, therefore, implies that the eigenvectors corresponding to the smallest eigenvalues should be chosen. Here, we have proven that the minimizing eigenvectors of $\Phi$ is a local minimum for both Eq.~(\ref{eq:obj_1}) and (\ref{eq:obj_2}).
\end{proof}
\end{appendices}

\begin{appendices}
\section{NMI Calculation}
\label{app:NMI_computation}

If we let $U$ and $L$ be two clustering assignments, NMI can be calculated with 
    \begin{align}
        NMI(L,U) = \frac{I(L,U)}{\sqrt{H(L)H(U)}},
        \label{eq:nmi_equation}
    \end{align}
    
where $I(L,U)$ is the mutual information between $L$ and $U$, and $H(L)$ and $H(U)$ are the entropies of $L$ and $U$ respectively. 
\end{appendices}

\begin{appendices}
\section{Proof for Corollary~\ref{corollary:combine_kernels}}
\label{app:combined_kernels}

\begin{proof}
    The optimization of Eq. (\ref{eq:obj_1}) using a conic combination of $m$ kernels becomes 
    
    \begin{equation}
        \underset{W}{\min} - \Tr \left( \Gamma 
        [\mu_1 K_1 + \mu_2 K_2 + ... + \mu_m K_m]
        \right) 
        \hspace{0.4cm} \text{s.t.} \hspace{0.2cm} W^T W = I.
    \end{equation}    
   The trace term can be separated into dividual terms where 
     \begin{equation}
        \underset{W}{\min}
        - \Tr ( \mu_1 \Gamma K_1 )
        - \Tr ( \mu_2 \Gamma K_2 )
        - ... 
        - \Tr ( \mu_m \Gamma K_m )
        \hspace{0.4cm} \text{s.t.} \hspace{0.2cm} W^T W = I.
    \end{equation}       
    Therefore, the Lagrangian can be written as
     \begin{equation}
        \mathcal{L} = 
        - \Tr ( \mu_1 \Gamma K_1 )
        - \Tr ( \mu_2 \Gamma K_2 )
        - ... 
        - \Tr ( \mu_m \Gamma K_m )
        - m Tr(\Lambda[W^T W - I]).
    \end{equation}          
    
    From Lemma~\ref{basic_lemma}, we have shown that the gradient of the Lagrangian becomes
      \begin{equation}
        \nabla_W \mathcal{L} = 
        \left[ - \mu_1 \Phi_1 - \mu_2 \Phi_2 - ... - \mu_m \Phi_m
        \right ] W
        - m W \Lambda,
    \end{equation}          
    where each $\Phi_i$ is the $\Phi$ matrix corresponding to each kernel. Setting the gradient to 0, it yields the relationship
    \begin{equation}
        \frac{1}{m}\left[ - \mu_1 \Phi_1 - \mu_2 \Phi_2 - ... - \mu_m \Phi_m
        \right ] W
        = W \Lambda,
    \end{equation}             
    Therefore, optimizing a conic combination of kernels for Eq.~(\ref{eq:obj_1}) is equivalent to using a conic combination of the corresponding $\Phi$s with the same coefficients.
\end{proof}
\end{appendices}
\begin{appendices}
\section{An Overview on HSIC}
\label{app:about_hsic}

Proposed by \citet{gretton2005measuring}, the Hilbert Schmidt Independence Criterion (HSIC) is a statistical dependence measure between two random variables. HSIC is similar to mutual information (MI) because given two random variables $X$ and $Y$, they both measure the distance between the joint distribution $P_{X,Y}$ and the product of their individual distributions $P_X P_Y$. While MI uses KL-divergence to measure this distance, HSIC uses Maximum Mean Discrepancy~\cite{gretton2012kernel}. Therefore, when HSIC is zero, or $P_{X,Y}=P_X P_Y$, it implies independence between $X$ and $Y$. Similar to MI, HSIC score increases as $P_{X,Y}$ and $P_X P_Y$ move away from each other, thereby also increasing their dependence. Although HSIC is similar to MI in its ability to measure dependence,  it is easier to compute as it removes the need to estimate the joint distribution. 


Formally, given a set of $N$ i.i.d. samples $\{(x_1,y_1),...,(x_N,y_N)\}$ drawn from a joint distribution $P_{X,Y}$. Let $X \in \mathbb{R}^{N \times d}$ and $Y  \in \mathbb{R}^{N \times c}$ be the corresponding sample matrices where $d$ and $c$ denote the dimensions of the datasets. We denote by $K_X,K_Y \in \mathbb{R}^{N \times N}$ the kernel matrices with entries $K_{X_{i,j}}=k_X(x_i,x_j)$ and $K_{Y_{i,j}} = k_Y(y_i,y_j)$, where $k_X: \mathbb{R}^d \times \mathbb{R}^d \rightarrow \mathbb{R}$ and $k_Y: \mathbb{R}^c \times \mathbb{R}^c \rightarrow \mathbb{R}$ represent kernel functions. Furthermore, let $H$ be a centering matrix defined as $H=I_n - \frac{1}{n} \textbf{1}_n\textbf{1}_n^T$ where $\textbf{1}_n$ is a column vector of ones. HSIC is computed empirically with
\begin{equation}
    \mathbb{H}(X,Y) = \frac{1}{(n-1)^2} \Tr(K_X H K_Y H).
    \label{eq:emprical_hsic}
\end{equation}



\end{appendices}
\begin{appendices}
\section{Proof for Proposition~\ref{thm:linear_combinations_of_kernels}}
\label{app:linear_combinations_of_kernels}

\begin{proof}
For a kernel to belong to the ISM family, it must satisfy the following 3 conditions. 
\begin{itemize}
  \item The kernel function must be twice differentiable.
  \item The kernel function can be written in terms of $f(\beta)$.
  \item The kernel matrix from $f(\beta)$ must be symmetric positive semi-definite.
\end{itemize}

To satisfy the 1st condition, given a kernel $K$ that is a conic combination of $n$ ISM kernels where 
    \begin{align}
        K = \mu_1 K_1 + \mu_2 K_2 + ... + \mu_n K_n.
        \label{eq:sum_of_K}
    \end{align}
    Since each kernel $K_i$ is twice differentiable, the conic combination is still twice differentiable. Therefore, $K$ is a twice differentiable function. 
    
    To satisfy the 2nd condition, given a kernel $K$ from Eq.~(\ref{eq:sum_of_K}) where each kernel is from the ISM family, the trivial case of when $\beta = a(x_i, x_j)W W^t b(x_i, x_j)$ is defined identically between kernels:
     \begin{align}
        K = \mu_1 f_1(\beta) + \mu_2 f_2(\beta) + ... + \mu_n f_n(\beta).
        \label{eq:sum_of_f_beta}
    \end{align}   
   From Eq.~(\ref{eq:sum_of_f_beta}), it is obvious that $K$ itself can also be written in terms of $\beta$.  However, in the cases where the functions $a(x_i,x_j)$ and $b(x_i,x_j)$ are defined differently, $\beta$ must be defined differently. Here, we define the following
    \begin{align}
        a(x_i,x_j) = \textbf{Diag}([a_1(x_i,x_j),a_2(x_i,x_j),..,a_n(x_i,x_j)])\\
        b(x_i,x_j) = \textbf{Diag}([b_1(x_i,x_j),b_2(x_i,x_j),..,b_n(x_i,x_j)])\\
        W = \textbf{Diag}([W, W, ..., W])\\
        W^T = \textbf{Diag}([W^T, W^T, ..., W^T])\\
        \beta = \textbf{Diag}([a_1^T W W^T b_1, a_2 ^T W W^T b_2, ..., a_n^T W W^T b_n])
    \end{align}     
    where \textbf{Diag} puts the element of the vector on the diagonal of a matrix with both the upper and lower triangle as 0s. Given $\beta$ is a matrix, each kernel function can always multiply $\beta$ by a one-hot vector on both sides to choose the appropriate sub-$\beta$ value. Therefore, the joint kernel $K$ can always be written in terms of $\beta$.

   For the 3rd condition, we know that conic combinations of symmetric positive semi-definite matrices are still symmetric positive semi-definite.
\end{proof}
\end{appendices}

\begin{appendices}
\section{Proof of Theorem~\ref{thm:convergence}}

\begin{proof}
The original ISM leverages Bolzano-Weierstrass theorem to prove that a sequence generated using the Gaussian kernel is bounded,
therefore ISM has a convergent subsequence. Since the generalized ISM extends the guarantee to other kernels, here we demonstrate that the extension of ISM to other kernels does not have any effect on the convergence guarantee.

ISM over arbitrary kernels solves an optimization problem over the Grassmannian manifold $G(n, d)$, as parametized by the subspace $W W^T$. The Grassmannian manifold is a quotient of the Stiefel Manifold $G(n, d) = V(n, d) / O(n)$. The Grassmann manifold inherits compactness and an induced metric from the Stiefel manifold. For metric spaces, compact and sequentially compact topological spaces are equivalent. Therefore, sequences $\{W W^T\}^k$ will have convergent subsequences. While $W$ may not converge (choice of frame), its subspace description will. Termination criteria in Algorithm 1 is independent of the frame $W$. 

\end{proof}
\label{app:convergence}
\end{appendices}

\begin{appendices}
\section{Convergence Criteria}

Since the objective is to discover a linear subspace, the rotation of the space does not affect the solution. Therefore, instead of constraining the solution on the Stiefel Manifold, the manifold can be relaxed to a Grassmann Manifold. This implies that Algorithm~\ref{alg:ism} can reach convergence as long as the columns space spanned by $W$ are identical. To identify the overlapping span of two spaces, we can append the two matrices into $\mathcal{W} = [W_k W_{k+1}]$ and observe the rank of $\mathcal{W}$. In theory, the rank should equal to $q$, however, a hard threshold on rank often suffers from numerical inaccuracies.

One approach is to study the principal angles (`angles between flats') between the subspaces spanned by $W_k$ and $W_{k+1}$. This is based on the observation that if the maximal principal angle $\theta_{\text{max}} = 0$, then the two subspaces span the same space. The maximal principal angle between subspaces spanned by $W_k$ and $W_{k+1}$ can be found by computing $U \Sigma V^T = W_k^T W_{k+1}$ \cite{knyazev2012principal}. The cosines of the principal angles between $W_k$ and $W_{k+1}$ are the singular values of $\Sigma$, thus $\theta_{\text{max}} = \cos^{-1}(\sigma_{\text{min}})$. Computation of $\theta_{\text{max}}$ requires two matrix multiplications to form $V \Sigma^2 V^T = (W_k^T W_{k+1})^T (W_k^T W_{k+1})$ and then a round of inverse iteration to find $\sigma_{\text{min}}^2$. Although this approach confirms the convergence definitively, in practice, we avoid this extra computation by using the convergence of eigenvalues (of $\Phi$) between iterations as a surrogate. Since eigenvalues are already computed during the algorithm, no additional computations are required. Although tracking eigenvalue of $\Phi$ for convergence is vulnerable to false positive errors, in practice, it works consistently well. Therefore, we recommend to use the eigenvalues as a preliminary check before defaulting to principal angles.

\label{app:convergence_criteria}
\end{appendices}

\end{document}